\documentclass{article}

\usepackage[final,main]{neurips_2026}

\usepackage[utf8]{inputenc}
\usepackage[T1]{fontenc}
\usepackage{microtype}
\usepackage{graphicx}
\graphicspath{{figures/}}
\usepackage{booktabs}
\usepackage{algorithm}
\usepackage{algorithmic}
\usepackage{url}
\usepackage{xurl}
\usepackage{hyperref}
\hypersetup{
  colorlinks=true,
  citecolor=blue,
  linkcolor=blue,
  urlcolor=blue,
  breaklinks=true
}
\usepackage{multirow}
\usepackage{makecell}
\usepackage{caption}
\usepackage{adjustbox}
\usepackage{tabularx}
\usepackage{amsmath}
\usepackage{amsfonts}
\usepackage{amssymb}
\usepackage{mathtools}
\usepackage{amsthm}
\usepackage[capitalize,noabbrev]{cleveref}
\usepackage{enumitem}
\usepackage{subcaption}
\usepackage{nicefrac}
\usepackage{xcolor}
\usepackage{placeins}

\theoremstyle{plain}

\theoremstyle{definition}

\theoremstyle{remark}

\title{Temporal Pair Consistency for Flow Matching}

\author{Chika Maduabuchi\\
  William \& Mary
  \And
  Jindong Wang\\
  William \& Mary\\
\texttt{jdw@wm.edu}}

\begin{document}

\maketitle

\begin{abstract}
Flow-matching-style generative models learn time-dependent vector fields, but standard objectives supervise each timestep independently, leaving shared-path temporal structure unused. We introduce \emph{Temporal Pair Consistency} (TPC), a simple training-time objective that couples velocity predictions at paired timesteps along the same probability path. TPC leaves the architecture, probability path, sampler, and inference-time computation unchanged. Under matched backbones, training budgets, solvers, and numbers of function evaluations, TPC improves sample quality across flow matching, OT-CFM, rectified flow, DiT-style backbones, and MeanFlow. On CIFAR-10, TPC improves FM from 6.35 to 3.19 FID and OT-CFM from 3.58 to 2.90 FID at the same NFE. Mechanism diagnostics show lower measured gradient variance, higher paired-gradient correlation, and reduced temporal roughness, while random and local pairing controls are substantially weaker. Higher-resolution ImageNet experiments further show consistent gains for DiT-XL/2 at 256 and 512 resolution and for MeanFlow at 256 resolution, all at matched inference cost.
\end{abstract}
\begin{figure}[H]
  \centering
  \includegraphics[width=\textwidth]{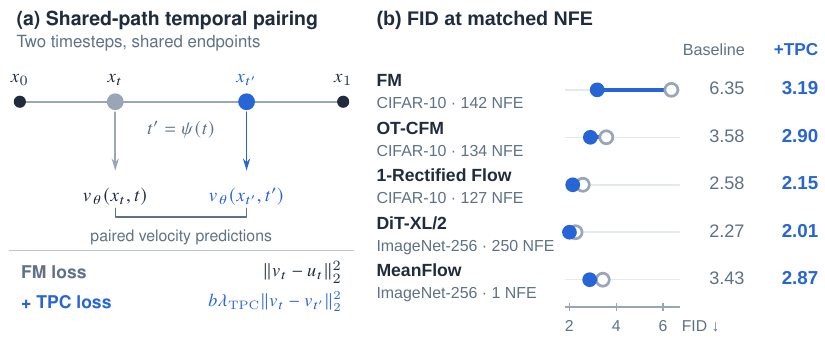}
  \caption{\textbf{Temporal Pair Consistency (TPC) at a glance.} (a) TPC couples velocity predictions at two timesteps on the same conditional probability path, adding a gated consistency penalty to the standard flow-matching objective. (b) TPC improves FID across five representative settings at matched inference NFE within each comparison. The architecture and sampler remain unchanged; CIFAR-10 TPC training requires approximately 1.68--1.70$\times$ baseline wall time.}
  \label{fig:tpc_teaser}
\end{figure}

\section{Introduction}
\label{sec:introduction}

Continuous-time generative models have become a central framework for high-fidelity image synthesis, spanning diffusion models (DMs), flow matching (FM), rectified flow (RF), and related deterministic or stochastic probability-flow formulations.
Diffusion models have achieved strong performance across a wide range of benchmarks, benefiting from stable training objectives and well-understood stochastic dynamics~\cite{NEURIPS2020_4c5bcfec,Rombach_2022_CVPR,10.1145/3626235}, but often require long sampling trajectories and substantial inference-time computation.
In parallel, deterministic formulations based on ordinary differential equations, including flow matching~\citep{lipman2023flow} and rectified flow~\citep{liu2023flow}, have shown that carefully designed probability paths can provide flexible training objectives and competitive sample quality~\cite{li2023selfconsistent}.

Despite this progress, standard flow-matching-style objectives usually supervise each timestep independently.
For a sampled endpoint pair and timestep, the model is trained to match the corresponding velocity target, but predictions at other timesteps along the same probability path are not explicitly coupled.
This leaves a simple source of temporal structure unused: different timesteps on the same path share the same endpoint randomness and often exhibit correlated training signals.
Empirically, we observe that independent timestep training can lead to noisy training dynamics and temporally uneven velocity fields, especially in regimes where sample quality is sensitive to the learned trajectory.
The goal of this paper is to test whether directly coupling such shared-path evaluations during training improves the learned vector field under controlled, matched inference settings.

We introduce \emph{Temporal Pair Consistency} (TPC), a lightweight training-time objective for flow matching.
Given two timesteps sampled along the same probability path, TPC adds a quadratic consistency term between the corresponding velocity predictions.
The method is deliberately simple: it does not change the model architecture, probability path, training target, sampler, or inference-time computation.
Instead, it augments the training objective with a shared-path temporal pairing term, encouraging the learned vector field to produce compatible predictions at coupled points on the same path.
This makes TPC easy to add to existing FM and RF implementations, while preserving the same inference procedure used by the original model.

We study TPC as an empirical, matched-cost training modification. Across all main comparisons, the TPC-augmented model uses the same architecture, probability path, sampler, training budget, and number of function evaluations as its baseline. This design isolates whether shared-path temporal pairing improves the learned vector field without changing inference. We further use pairing controls, coupling-strength ablations, and training-dynamics diagnostics to test whether the gains are tied to structured temporal coupling rather than arbitrary regularization.

\paragraph{Contributions.}
\vspace{-0.5em}
Our contributions are threefold:
\vspace{-0.5em}
\begin{enumerate}[leftmargin=*, label=(\roman*), topsep=0pt, itemsep=0pt, parsep=0pt]
    \item We introduce \emph{Temporal Pair Consistency} (TPC), a simple shared-path temporal pairing objective for flow matching. TPC couples velocity predictions at paired timesteps along the same probability path while leaving the architecture, probability path, sampler, and inference-time NFE unchanged.
    
    \item We provide a controlled empirical study of TPC's training behavior. Under matched settings, TPC lowers measured gradient variance, increases paired-gradient correlation, and reduces temporal roughness; random and local pairing controls are weaker, and coupling-strength ablations reveal the expected over-regularization regime.
    
    \item We show that TPC improves sample quality at matched inference cost across multiple FM-style settings, including FM, OT-CFM, RF, DiT-style diffusion backbones, and MeanFlow. The gains persist from CIFAR-10 and ImageNet benchmarks to higher-resolution $256\times256$ and $512\times512$ evaluations.
\end{enumerate}

\section{Related Work}
\label{sec:related_work}

\paragraph{Temporal and trajectory regularization in continuous-time models.}
TPC is related to prior work on temporal smoothness, consistency regularization, and trajectory regularization in continuous-time generative models. Several approaches regularize learned ODE dynamics by penalizing path length, kinetic energy, Jacobians, or other properties of the velocity field, often with the goal of improving numerical stability or invertibility in continuous normalizing flows~\cite{10.5555/3327757.3327764,DBLP:conf/iclr/GrathwohlCBSD19,kidger2022neuraldifferentialequations,yang2025consistency}. Related ideas also appear in diffusion models, where smoothness of the score or drift over time is encouraged through architectural choices, noise schedules, or discretization design~\cite{song2021scorebased,NEURIPS2022_a98846e9}. Recent work on flow-based generation further studies path design and straight-trajectory learning through rectified, latent-augmented, or variational formulations of flow matching~\cite{liu2023flow,ma2025learningstraightflowsvariational}.

TPC differs from these approaches in where the modification is applied. Rather than changing the probability path, sampler, architecture, or ODE solver, TPC adds a training-time paired consistency term between velocity predictions evaluated at two timesteps along the same conditional probability path. The objective is therefore not a generic Jacobian or path-length penalty on the vector field, but a shared-path coupling of existing flow-matching evaluations. This makes TPC easy to add to standard flow-matching-style training while keeping the inference procedure unchanged. In our experiments, this distinction is tested directly through random and local pairing controls, which isolate whether improvements come from arbitrary temporal smoothing or from the structured shared-path pairing used by TPC.

\paragraph{Consistency regularization and consistency-based generative models.}
TPC is also connected to consistency regularization methods in supervised and self-supervised learning, where predictions are encouraged to agree under input perturbations, augmentations, or stochastic transformations~\cite{laine2017temporal,NEURIPS2020_06964dce}. In generative modeling, consistency models and consistency-style flow methods enforce agreement across noise levels or timesteps to enable few-step or one-step generation~\cite{lu2025simplifying,issenhuth2024improving}. In particular, Consistency Flow Matching (Consistency-FM) enforces velocity consistency along model-induced flow trajectories to facilitate straight or few-step generation~\cite{yang2025consistency}.

TPC has a different goal: it couples velocity predictions at two timesteps on the same prescribed conditional probability path, reusing the same endpoint randomness. It retains the original flow-matching regression target and does not modify the inference sampler or NFE. The objectives are complementary, as supported by our matched two-NFE composition experiment with Consistency-FM. Few-step consistency methods remain useful reference points, though direct matched-NFE comparisons are uncommon; our composition experiment provides one such comparison. Our main comparisons therefore emphasize matched architecture, path, solver, training budget, and NFE, while the matched two-NFE Consistency-FM composition directly tests whether the objectives complement one another.

\paragraph{Relation to flow matching, rectified flow, and modern flow pipelines.}
Flow matching~\cite{lipman2023flow}, rectified flow~\cite{liu2023flow}, stochastic-interpolant methods~\cite{albergo2023building}, and recent temporally structured flow constructions~\cite{ifriqi2025flowceptiontemporallyexpansiveflow} improve continuous-time generative modeling primarily through the design of probability paths, interpolation schemes, or trajectory geometry. These choices determine the target velocity field and the dynamics that are integrated at sampling time.

TPC is complementary to this line of work. It does not propose a new probability path or a new solver; instead, it modifies the training objective by coupling two velocity predictions drawn from the same existing path sampler. This allows the same idea to be applied on top of different FM-style formulations, including standard flow matching, OT-CFM, rectified flow, DiT-style diffusion backbones trained with flow objectives, and average-velocity matching frameworks such as MeanFlow. This positions TPC as a training-objective modification whose role is evaluated under matched inference settings in Sections~\ref{sec:method}--\ref{experiments}.

\section{Method}
\label{sec:method}
Flow matching learns a time-indexed vector field that transports a simple reference distribution into the target data distribution by integrating an ordinary differential equation. We introduce \emph{Temporal Pair Consistency} (TPC), a training-time objective that couples velocity predictions at paired timesteps along the same probability path. TPC does not change the probability path, model architecture, sampler, or inference-time computation. It simply adds a shared-path consistency term during training. We first review conditional flow matching in Section~\ref{sec:flowmatching_formulation}, then define shared-path temporal pairing in Section~\ref{sec:tpcvar}, describe fixed and learned pairing rules in Section~\ref{sec:pairing_mechanisms}, and give the full training objective in Sections~\ref{sec:temporal_consistency_objective}--\ref{sec:training_objective}. Figure~\ref{fig:tpc_method} summarizes the difference between standard flow matching, TPC training, and unchanged inference.

\begin{figure}[H]
    \centering
    \includegraphics[width=\textwidth]{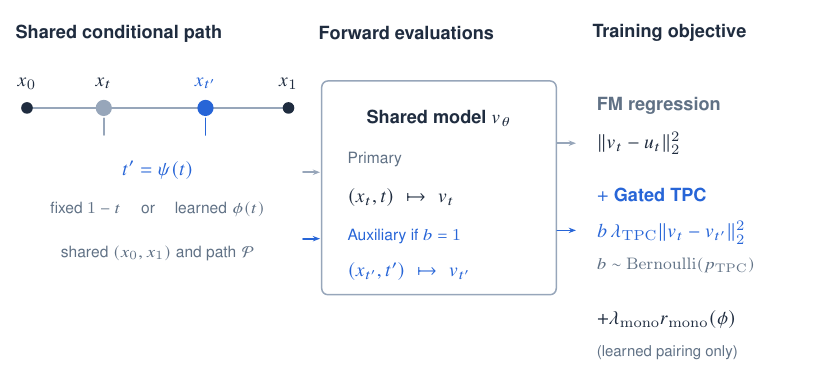}
    \caption{
    \textbf{Temporal Pair Consistency (TPC) training pipeline.} Two timesteps $t$ and $t'=\psi(t)$ are evaluated along the same conditional probability path using shared endpoints and network parameters. The primary prediction retains standard flow-matching supervision, while a Bernoulli-gated auxiliary evaluation contributes the TPC consistency loss. Pairing can be fixed or learned, with optional monotonicity regularization for learned pairing only. Inference uses the original architecture, probability path, sampler, solver, and NFE.
    }
    \label{fig:tpc_method}
\end{figure}

\subsection{Conditional Flow Matching}
\label{sec:flowmatching_formulation}

Let $x_0 \sim p_0$ denote a base distribution, such as $\mathcal{N}(0,I)$, and let $x_1 \sim p_1$ denote the data distribution. A conditional probability path is induced by a map
\(
x_t = \Phi_t(x_0,x_1), \quad t\in[0,1],
\)
with associated conditional velocity
\(
u_t(x_t)=\partial_t \Phi_t(x_0,x_1).
\)
Flow matching parameterizes a velocity field $v_\theta(x,t)$ and minimizes
\begin{equation}
\label{eq:fm}
\mathcal{L}_{\mathrm{FM}}(\theta)
=
\mathbb{E}_{t\sim\rho,\,(x_0,x_1)}
\bigl\| v_\theta(x_t,t) - u_t(x_t) \bigr\|_2^2,
\end{equation}
where $\rho$ is typically uniform on $[0,1]$. Eq.~\eqref{eq:fm} supervises the velocity field at the sampled timestep $t$. In standard training, predictions at other timesteps along the same endpoint pair are not explicitly coupled.

\subsection{Temporal Pairing for Shared-Path Consistency}
\label{sec:tpcvar}

TPC uses the endpoint pair already sampled by flow matching to construct a second point on the same probability path. For $t\sim\rho$, let $t'=\psi(t)$ be an auxiliary timestep. Using the same endpoints $(x_0,x_1)$, the path sampler produces $(x_t,u_t)=\mathcal{P}(x_0,x_1,t)$ and $(x_{t'},u_{t'})=\mathcal{P}(x_0,x_1,t')$.
The corresponding model predictions are $v_t=v_\theta(x_t,t)$ and $v_{t'}=v_\theta(x_{t'},t')$.
The TPC penalty is the squared discrepancy between these two velocity predictions,
\begin{equation}
\label{eq:tpc_penalty}
\ell_{\mathrm{TPC}}(t,t')
=
\bigl\|v_t-v_{t'}\bigr\|_2^2.
\end{equation}
This term encourages the learned field to produce compatible velocity predictions at two coupled states generated from the same endpoint randomness. Importantly, the target velocity in Eq.~\eqref{eq:fm}, the probability path, and the sampling solver are unchanged. The empirical sections test whether this shared-path coupling improves sample quality and training diagnostics under matched inference settings.

\subsection{Pairing Mechanisms}
\label{sec:pairing_mechanisms}

TPC is defined by a pairing operator $\psi:[0,1]\to[0,1]$ that maps a primary timestep $t$ to an auxiliary timestep $t'=\psi(t)$. We study two practical choices.

\paragraph{Fixed symmetric pairing.}
The core, parameter-free TPC method uses the fixed map
\begin{equation}
\label{eq:fixed_pairing}
\psi_{\mathrm{fix}}(t)=1-t.
\end{equation}
This rule couples early and late portions of the same probability path and cannot collapse to self-pairing.

\paragraph{Learned monotone pairing.}
As an optional adaptive extension, we consider a low-capacity learned map $\phi:[0,1]\to[0,1]$, defining $\psi_{\mathrm{learn}}(t)=\phi(t)$. To preserve the temporal ordering structure of the path, we parameterize $\phi$ as a monotone function,
\begin{equation}
\label{eq:learned_pairing}
\phi(t)
=
\sigma\!\left(\sum_{i=1}^H a_i\,\sigma(t+b_i)+c\right),
\qquad
 a_i=\mathrm{softplus}(\tilde a_i)>0,
\end{equation}
where $\sigma$ is the sigmoid, $b_i\in\mathbb{R}$ controls the transition location of each hidden unit, and $c\in\mathbb{R}$ is a global offset. The positive coefficients make each hidden unit contribute non-negatively to the slope before the outer sigmoid, yielding a compact monotone parameterization.

For learned pairing only, monotonicity is encouraged using a smooth grid-based penalty rather than a hard indicator. For grid points $\{g_k=k/K\}_{k=1}^K$, with $K=32$ in our experiments, we use
\begin{equation}
\label{eq:mono_smooth}
 r_{\mathrm{mono}}(\phi)
=
\sum_{k=1}^{K-1}
\mathrm{softplus}\!\left(\phi(g_k)-\phi(g_{k+1})\right).
\end{equation}
For fixed pairing, we set $r_{\mathrm{mono}}=0$; learned monotone pairing is an optional adaptive extension and gave the strongest result.

\subsection{Temporal Consistency Objective}
\label{sec:temporal_consistency_objective}

For a sampled path and timestep pair, draw $b\sim\mathrm{Bernoulli}(p_{\mathrm{tpc}})$ and define the stochastically gated TPC penalty
\begin{equation}
\label{eq:gated_tpc}
\tilde{\ell}_{\mathrm{TPC}}(t,t')
=
b\,\ell_{\mathrm{TPC}}(t,t').
\end{equation}
The single-sample loss is then
\begin{equation}
\label{eq:single_sample_objective}
\ell_{\mathrm{sample}}(\theta;\psi)
=
\bigl\|v_\theta(x_t,t)-u_t\bigr\|_2^2
+
\lambda_{\mathrm{tpc}}\tilde{\ell}_{\mathrm{TPC}}(t,t')
+
\lambda_{\mathrm{mono}} r_{\mathrm{mono}},
\end{equation}
where $\lambda_{\mathrm{tpc}}\ge 0$ controls the coupling strength and $\lambda_{\mathrm{mono}}\ge 0$ is used only for learned pairing, with $r_{\mathrm{mono}}=r_{\mathrm{mono}}(\phi)$ as defined in Eq.~\eqref{eq:mono_smooth}; for fixed pairing, $r_{\mathrm{mono}}=0$. The gate probability $p_{\mathrm{tpc}}$ controls how often coupling is applied. Eq.~\eqref{eq:single_sample_objective} is a training loss, not a modified sampler or probability path; inference uses the baseline solver and NFE.

\subsection{Full Training Objective}
\label{sec:training_objective}

Combining flow matching, stochastic temporal pairing, and the optional monotonicity penalty gives the objective
\begin{equation}
\label{eq:objective}
\mathcal{L}(\theta;\psi)=\mathbb{E}_{t\sim\rho,\,(x_0,x_1),\,b\sim\mathrm{Bernoulli}(p_{\mathrm{tpc}})}\left[
\begin{aligned}
&\bigl\|v_\theta(x_t,t)-u_t\bigr\|_2^2\\
&+ \lambda_{\mathrm{tpc}}\tilde{\ell}_{\mathrm{TPC}}(t,t')\\
&+ \lambda_{\mathrm{mono}} r_{\mathrm{mono}}
\end{aligned}
\right],
\end{equation}
where $t'=\psi(t)$ and both path points use the same endpoint pair $(x_0,x_1)$ and path sampler $\mathcal{P}$. Thus the expected objective is exactly the expectation of the single-sample loss in Eq.~\eqref{eq:single_sample_objective}; it adds no inference-time computation. Algorithmic details are provided in Appendix~\ref{app:tpc_algorithm}.

\section{Experiments}
\label{experiments}

We organize the empirical evaluation around three questions: whether TPC improves sample quality under matched inference cost, whether the gains are associated with measurable changes in training dynamics, and whether the same training-time modification transfers beyond the original low-resolution FM setting. We first describe the matched setup in Section~\ref{sec:exp_setup}, then report unconditional FM and RF results in Section~\ref{sec:unconditional_results}, mechanism diagnostics and pairing controls in Section~\ref{sec:mechanism_diagnostics}, noise-augmented ImageNet evaluations and higher-resolution DiT/MeanFlow extensions in Section~\ref{sec:modern_flow_pipelines}, and ablations over pairing and coupling strength in Section~\ref{sec:ablations}. Figure~\ref{fig:matched_cost_teaser} summarizes the matched-cost FID improvements across the main experimental settings.

\begin{figure}[t]
    \centering
    \includegraphics[width=0.92\linewidth]{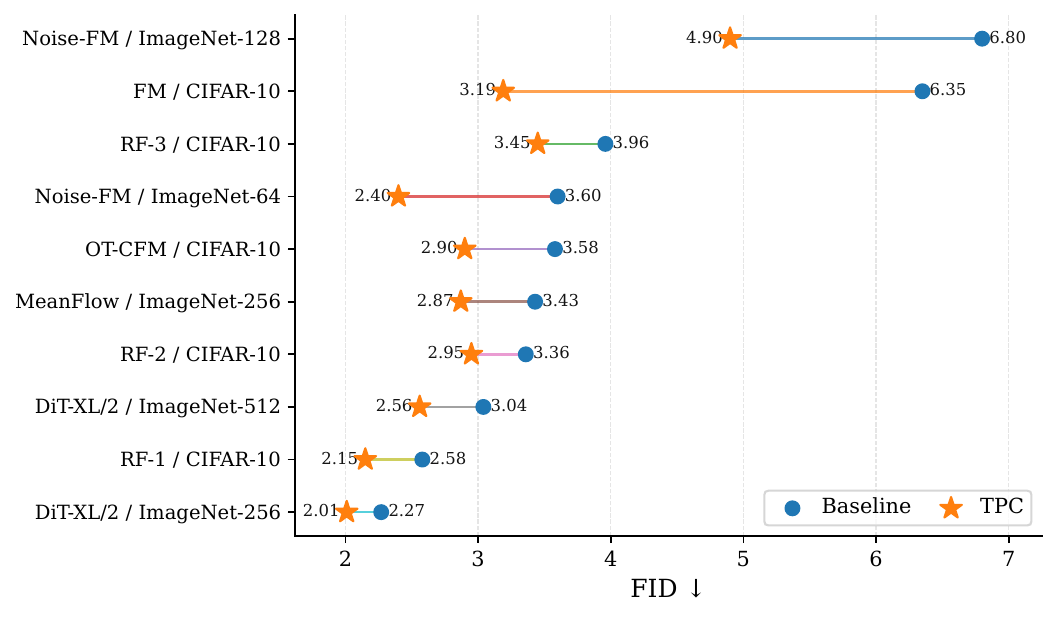}
    \caption{
    \textbf{TPC improves sample quality at matched inference cost.}
    Each row compares a baseline model with its TPC-augmented counterpart using the same architecture,
    probability path, sampler, and number of function evaluations. Endpoint labels show FID values;
    exact NFEs are reported in the corresponding tables. TPC consistently shifts FID lower across
    standard FM, OT-CFM, rectified flow, noise-augmented ImageNet flow matching, DiT backbones,
    and MeanFlow without changing inference-time computation.
    }
    \label{fig:matched_cost_teaser}
\end{figure}

\subsection{Setup}
\label{sec:exp_setup}

We evaluate Temporal Pair Consistency (TPC) on standard image-generation benchmarks following established protocols in diffusion and continuous-time generative modeling~\cite{NEURIPS2020_4c5bcfec,song2021scorebased,lipman2023flow,liu2023flow}. The main experiments use CIFAR-10~\cite{krizhevsky2009learning} and ImageNet~\cite{5206848} at $32{\times}32$, $64{\times}64$, and $128{\times}128$, with additional higher-resolution extensions at $256{\times}256$ and $512{\times}512$. Unless otherwise stated, baseline and TPC-augmented models use the same backbone, probability path, solver, training budget, and sampling budget. This matched setting is intentional: our goal is to isolate the effect of the training objective rather than confound it with architecture changes, altered probability paths, or increased inference-time computation.

Matched steps do not imply matched training compute: CIFAR-10 TPC uses $1.68$--$1.70\times$ relative wall time and $1.04$--$1.05\times$ peak memory, but leaves inference architecture, solver, and NFE unchanged. Equal- and near-equal-wall-time controls show learned TPC reaches FID $3.19$, versus $5.48$ for longer-trained FM and $5.72$ for second-timestep supervision without consistency (Appendix Table~\ref{tab:matched_compute_controls}); this distinguishes step, wall-clock, and inference-cost matching.

We report negative log-likelihood (NLL)~\cite{lipman2023flow} where applicable, Fr\'echet Inception Distance (FID)~\cite{NIPS2017_8a1d6947} for sample quality, Inception Score (IS)~\cite{NIPS2016_8a3363ab} and recall where used by the baseline protocol, and the number of function evaluations (NFE)~\cite{DBLP:conf/iclr/GrathwohlCBSD19} to document the inference cost. NFE is held fixed when comparing TPC to each matched baseline; the central empirical claim is that TPC improves sample quality at matched inference cost. For flow-matching models, samples are generated by solving the learned generative ODE from noise to data using the same solver settings as the corresponding baseline: Flow Matching models use adaptive ODE solvers with absolute and relative tolerances set to $10^{-5}$~\cite{lipman2023flow}, while Rectified Flow models use adaptive Runge--Kutta (RK45) solvers for full simulation and fixed-step Euler solvers for one-step and distillation settings~\cite{liu2023flow}. All FID, IS, and recall metrics are computed using standard evaluation protocols over $50$k samples.

To diagnose the empirical mechanism, we measure gradient variance, paired-gradient correlation, and temporal roughness under the same CIFAR-10 backbone, training budget, and solver settings. These diagnostics are used only as empirical measurements of how shared-path temporal pairing changes training dynamics. We also include random and local pairing controls, coupling-strength ablations, a closest consistency-style reference point, high-resolution DiT-style extensions~\cite{10377858}, MeanFlow extensions~\cite{geng2025mean}, and qualitative comparisons under matched sampling settings. Family-specific implementation details for applying TPC to FM, OT-CFM, rectified flow, DiT-style velocity models, noise-augmented FM, and MeanFlow are provided in Appendix~\ref{app:tpc_family_implementation}.

\FloatBarrier

\subsection{Unconditional Image Generation}
\label{sec:unconditional_results}

\paragraph{Flow Matching}
\begin{table*}[!t]
\centering
\caption{\textbf{Unconditional diffusion and flow-matching results on CIFAR-10, ImageNet} $32\times32$, ImageNet $64\times64$, and ImageNet $128\times128$. Blank entries indicate methods that do not report the corresponding unconditional setting. The ImageNet-$128{\times}128$ panel compares FM w/ OT and TPC-FM. Table~\ref{tab:conditional_imagenet_context} reports the separate noise-augmented conditional ImageNet evaluations, with ImageNet-$128{\times}128$ in panel (a) and ImageNet-$64{\times}64$ in panel (b).}
\label{tab:flow_matching}

% ========================= LEFT TABLE =========================
\begin{minipage}{0.60\textwidth}
\centering
{\small
\setlength{\tabcolsep}{3pt}%
\begin{adjustbox}{max width=\linewidth}
\begin{tabular}{@{}lccccccccc@{}}
\toprule
& \multicolumn{3}{c}{\textbf{CIFAR-10}}
& \multicolumn{3}{c}{\textbf{ImageNet 32$\times$32}}
& \multicolumn{3}{c}{\textbf{ImageNet 64$\times$64}} \\
\cmidrule(lr){2-4} \cmidrule(lr){5-7} \cmidrule(lr){8-10}
\textbf{Model}
& NLL$\downarrow$ & FID$\downarrow$ & NFE$\downarrow$
& NLL$\downarrow$ & FID$\downarrow$ & NFE$\downarrow$
& NLL$\downarrow$ & FID$\downarrow$ & NFE$\downarrow$ \\
\midrule

DDPM {\cite{NEURIPS2020_4c5bcfec}} & 3.12 & 7.48 & 274 & 3.54 & 6.99 & 262 & 3.32 & 17.36 & 264 \\

Score Matching {\cite{lipman2023flow}} & 3.16 & 19.94 & 242 & 3.56 & 5.68 & 178 & 3.40 & 19.74 & 441 \\

i-DODE {\cite{10.5555/3618408.3620191}}
& 2.56 & 11.20 & 162
& 3.69 & 10.31 & 138
& -- & -- & -- \\

ScoreFlow {\cite{lipman2023flow}} & 3.09 & 20.78 & 428 & 3.55 & 14.14 & 195 & 3.36 & 24.95 & 601 \\

FM w/ Diffusion {\cite{lipman2023flow}}
& 3.10 & 8.06 & 183
& 3.54 & 6.37 & 193
& 3.33 & 16.88 & 187 \\

FM w/ OT {\cite{lipman2023flow}}
& 2.99 & 6.35 & 142
& 3.53 & 5.02 & 122
& 3.31 & 14.45 & 138 \\

I-CFM {\cite{tong2024improving}}
& -- & 3.66 & 146
& -- & -- & --
& -- & -- & -- \\

OT-CFM {\cite{tong2024improving}}
& -- & 3.58 & 134
& -- & -- & --
& -- & -- & -- \\

Discrete FM {\cite{NEURIPS2024_f0d629a7}}
& -- & 3.63 & 1024
& -- & -- & --
& -- & -- & -- \\

V-RFM {\cite{guo2025variational}}
& -- & 3.58 & 1000
& -- & -- & --
& -- & -- & -- \\

\midrule
\textbf{TPC-FM (Ours)}
& \textbf{2.99} & \textbf{3.19} & \textbf{142}
& \textbf{3.53} & \textbf{4.22} & \textbf{122}
& \textbf{3.31} & \textbf{13.14} & \textbf{138} \\

\textbf{TPC + OT-CFM (Ours)}
& -- & \textbf{2.90} & \textbf{134}
& -- & -- & --
& -- & -- & -- \\
\bottomrule
\end{tabular}
\end{adjustbox}
} % end small
\end{minipage}
\hfill
% ========================= RIGHT TABLE WITH REFERENCES =========================
\begin{minipage}{0.34\textwidth}
\centering
{\small
\setlength{\tabcolsep}{2pt}%
\resizebox{.99\linewidth}{!}{
\begin{tabular}{@{}lcc@{}}
\toprule
\textbf{Model (IN 128 $\times$ 128)} & \textbf{NLL$\downarrow$} & \textbf{FID$\downarrow$} \\
\midrule
\makecell[l]{FM w/ OT {\tiny\cite{lipman2023flow}}}            & 2.90 & 20.9 \\
\midrule
\textbf{TPC-FM (Ours)}                                               & \textbf{2.90} & \textbf{18.6} \\
\bottomrule
\end{tabular}
}
}
\end{minipage}
\end{table*}

Table~\ref{tab:flow_matching} reports unconditional image generation results on CIFAR-10 and ImageNet at resolutions $32{\times}32$, $64{\times}64$, and $128{\times}128$. Across datasets and resolutions, TPC-FM improves the FM w/ OT baseline under matched sampling budgets. On CIFAR-10, TPC-FM reduces FID from $6.35$ to $3.19$ at the same NFE of $142$, while TPC + OT-CFM improves a stronger OT-CFM baseline from $3.58$ to $2.90$ at the same NFE of $134$. Similar matched-NFE gains are observed on ImageNet $32{\times}32$ and $64{\times}64$, and TPC-FM improves ImageNet $128{\times}128$ FID from $20.9$ to $18.6$ while matching likelihood performance. These results support the intended empirical claim: TPC is a training-time modification that improves sample quality without changing the inference architecture, probability path, solver, or NFE.

\paragraph{Rectified Flow}
\begin{table*}[!t]
\centering
\caption{
\textbf{Matched rectified-flow results and CIFAR-10 mechanism diagnostics}.
Panel (a) reports full-simulation RF comparisons under matched adaptive RK45 settings.
Panel (b) reports pairing controls and diagnostic measurements under matched CIFAR-10 settings.
Panel (c) reports coupling-strength diagnostics for learned TPC-FM; Eff. wt. denotes $p_{\mathrm{tpc}}\lambda_{\mathrm{tpc}}$.
}
\label{tab:rf_mechanism_combined}
\scriptsize
\setlength{\tabcolsep}{2.0pt}

\begin{minipage}[t]{0.38\textwidth}
\centering
\textbf{(a) Rectified flow}\\[0.25em]
\resizebox{\linewidth}{!}{
\begin{tabular}{lcccc}
\toprule
\textbf{Method} &
\textbf{NFE$\downarrow$} &
\textbf{IS$\uparrow$} &
\textbf{FID$\downarrow$} &
\textbf{Rec.$\uparrow$} \\
\midrule
1-Rectified Flow & 127 & 9.60 & 2.58 & 0.57 \\
2-Rectified Flow & 110 & 9.24 & 3.36 & 0.54 \\
3-Rectified Flow & 104 & 9.01 & 3.96 & 0.53 \\
\midrule
\textbf{TPC-1RF} & 127 & \textbf{9.78} & \textbf{2.15} & \textbf{0.60} \\
\textbf{TPC-2RF} & 110 & 9.42 & 2.95 & 0.58 \\
\textbf{TPC-3RF} & 104 & 9.18 & 3.45 & 0.55 \\
\bottomrule
\end{tabular}
}
\end{minipage}
\hfill
\begin{minipage}[t]{0.60\textwidth}
\centering
\textbf{(b) Pairing controls and diagnostics}\\[0.25em]
\resizebox{\linewidth}{!}{
\begin{tabular}{lccccc}
\toprule
\textbf{Method} &
\textbf{FID$\downarrow$} &
\textbf{NFE$\downarrow$} &
\textbf{Grad var/FM$\downarrow$} &
\textbf{Grad corr$\uparrow$} &
\textbf{Rough./FM$\downarrow$} \\
\midrule
FM w/ OT & 6.35 & 142 & $1.00 \pm 0.03$ & $0.18 \pm 0.02$ & $1.00 \pm 0.04$ \\
FM + random pairing & 5.95 & 142 & $0.92 \pm 0.03$ & $0.25 \pm 0.03$ & $0.91 \pm 0.03$ \\
FM + local pairing & 5.20 & 142 & $0.82 \pm 0.03$ & $0.39 \pm 0.03$ & $0.78 \pm 0.03$ \\
OT-CFM & 3.58 & 134 & $0.84 \pm 0.03$ & $0.31 \pm 0.03$ & $0.86 \pm 0.03$ \\
\midrule
\textbf{TPC-FM} & 3.19 & 142 & $0.66 \pm 0.02$ & $0.61 \pm 0.03$ & $0.58 \pm 0.03$ \\
\textbf{TPC + OT-CFM} & \textbf{2.90} & 134 & $\mathbf{0.58 \pm 0.02}$ & $\mathbf{0.67 \pm 0.03}$ & $\mathbf{0.51 \pm 0.03}$ \\
\bottomrule
\end{tabular}
}

\vspace{0.55em}

\textbf{(c) Coupling strength}\\[0.25em]
\resizebox{\linewidth}{!}{%
\begin{tabular}{ccccccc}
\toprule
$p_{\mathrm{tpc}}$ &
$\lambda_{\mathrm{tpc}}$ &
Eff. wt. &
\textbf{FID$\downarrow$} &
\textbf{Grad var/FM$\downarrow$} &
\textbf{Grad corr$\uparrow$} &
\textbf{Rough./FM$\downarrow$} \\
\midrule
0.00 & 0.00 & 0.000 & 6.350 & 1.00 & 0.18 & 1.00 \\
0.25 & 0.10 & 0.025 & 3.922 & 0.87 & 0.30 & 0.85 \\
0.50 & 0.25 & 0.125 & 3.885 & 0.74 & 0.47 & 0.70 \\
0.75 & 0.10 & 0.075 & \textbf{3.193} & 0.66 & 0.61 & 0.58 \\
0.25 & 0.50 & 0.125 & 4.767 & 0.69 & 0.56 & 0.61 \\
1.00 & 1.00 & 1.000 & 6.223 & \textbf{0.63} & \textbf{0.64} & \textbf{0.55} \\
\bottomrule
\end{tabular}}
\end{minipage}

\vspace{-0.25em}
\end{table*}

Table~\ref{tab:rf_mechanism_combined}(a) reports rectified-flow results under full-simulation adaptive RK45 sampling. Across all rectification depths, adding TPC during training improves FID and recall at identical NFE. TPC-1RF improves FID from $2.58$ to $2.15$ at the same NFE of $127$, TPC-2RF improves from $3.36$ to $2.95$ at NFE $110$, and TPC-3RF improves from $3.96$ to $3.45$ at NFE $104$. These results show that TPC improves the learned vector field under matched inference settings, without changing the solver or using additional sampling steps. The complete one-step and full-simulation RF results are provided in Appendix Table~\ref{tab:rf_benchmark_full_appendix}.

\subsection{Mechanism Diagnostics and Pairing Controls}
\label{sec:mechanism_diagnostics}

We next analyze whether the observed improvements are tied to structured temporal pairing rather than arbitrary regularization. All diagnostics in this subsection use the same CIFAR-10 backbone, training budget, and solver settings. Gradient variance and temporal roughness are normalized by the FM w/ OT baseline, while paired-gradient correlation measures how strongly the paired evaluations align during training.

Table~\ref{tab:rf_mechanism_combined}(b) shows a consistent ordering across both sample quality and diagnostics. Random pairing, which uses an independent $t'\sim U[0,1]$, gives only a modest improvement over FM, while local pairing is stronger but still below structured TPC. TPC-FM reduces measured gradient variance to $0.66\times$ FM, increases paired-gradient correlation from $0.18$ to $0.61$, and reduces temporal roughness to $0.58\times$ FM while improving FID from $6.35$ to $3.19$ at the same NFE. Adding TPC on top of OT-CFM further improves FID from $3.58$ to $2.90$ at the same NFE of $134$, indicating that TPC is complementary to a stronger FM-family baseline. Appendix Table~\ref{tab:training_phase_diagnostics} shows the same training-wide trend: relative gradient variance drops from $0.84\times$ to $0.66\times$, while paired-gradient correlation rises from $0.34$ to $0.61$.

For the linear OT conditional path, paired points share the same conditional displacement target, providing a basis for correlated training signals. Table~\ref{tab:rf_mechanism_combined}(c) shows that moving from effective coupling $0.000$ to $0.075$ lowers gradient variance from $1.00$ to $0.66$ and raises paired-gradient correlation from $0.18$ to $0.61$, alongside improved FID from $6.350$ to $3.193$. Increasing coupling further to $1.000$ changes variance only from $0.66$ to $0.63$ and correlation from $0.61$ to $0.64$, while FID worsens to $6.223$. This empirical pattern is consistent with a bias--variance trade-off: structured coupling can reduce variance and align training signals, but because the optimal marginal velocity field varies with state and time, excessive coupling can suppress genuine temporal variation and introduce fitting bias. These observations support this interpretation but do not establish a formal causal proof.

Appendix Table~\ref{tab:consistency_reference} reports a direct same-codebase composition experiment at matched 2-NFE inference cost. Fixed and learned TPC improve reproduced Consistency-FM FID from $5.34$ to $4.70$ and $4.29$, respectively. Learned TPC yields a $1.05$ FID improvement over the reproduction and beats the equal-wall-time longer-training control ($5.08$ FID at $1.70\times$ wall time) by $0.79$ FID.

\subsection{Extending TPC to Modern Flow Pipelines}
\label{sec:modern_flow_pipelines}

The preceding sections evaluate standard probability--flow sampling. To test whether the same training-time objective remains useful in more modern ImageNet-style pipelines, we also study noise-augmented flow matching with score-based denoising at sampling time, a design used by several high-resolution diffusion and flow-based models~\cite{dhariwal2021diffusion,song2021scorebased,zhai2025normalizing}.

We evaluate TPC as an \emph{orthogonal training-time addition} under this protocol: the baseline and TPC model use the same noise augmentation, denoising rule, architecture, and sampling settings. As in prior work, denoising is performed using the model-implied score of the noise-augmented probability flow, so the matched comparison isolates the effect of adding TPC to the flow-matching training objective.

We evaluate this setting on conditional ImageNet generation at $64{\times}64$ and $128{\times}128$. The most relevant comparison is between the matched FM baseline and TPC-FM under the same noise-augmented and denoised protocol. We also show diffusion~\cite{dhariwal2021diffusion,nichol2021improved} and consistency-model~\cite{10.5555/3618408.3619743,song2024improved} results from prior work for context; these use their reported settings and are not matched comparisons or primary evidence for TPC's contribution. The combined results are reported in Table~\ref{tab:conditional_imagenet_context}, with ImageNet-$128{\times}128$ in panel (a) and ImageNet-$64{\times}64$ in panel (b).

\paragraph{Noise specification.}
In these noise-augmented evaluations, we adopt additive Gaussian noise during training to match modern high-resolution flow-based pipelines~\cite{song2021scorebased,zhai2025normalizing}. Specifically, states along the probability path are perturbed as $x_t \leftarrow x_t + \varepsilon$, with $\varepsilon \sim \mathcal{N}(0,\sigma^2 I)$. The noise scale $\sigma$ is selected to balance denoising strength and flow fidelity and is chosen \emph{exclusively} to optimize FID under the baseline flow-matching setup. Unless otherwise stated, we use $\sigma=0.05$ for ImageNet-$64$ and $\sigma=0.15$ for ImageNet-$128$. These values are then held fixed when training TPC-FM, with no additional tuning, so the TPC comparison is not confounded by additional noise-hyperparameter tuning.
\begin{table}[!t]
\centering
\caption{
\textbf{Conditional ImageNet generation performance under noise-augmented training and score-based denoising.}
External diffusion and consistency-model entries are shown for context using their reported settings; the matched comparison is FM + noise + denoising versus TPC-FM + noise + denoising.
}
\label{tab:conditional_imagenet_context}
\scriptsize
\setlength{\tabcolsep}{2.0pt}

\begin{minipage}[t]{0.49\linewidth}
\centering
\textbf{(a) ImageNet $128{\times}128$}\\[0.25em]
\begin{adjustbox}{max width=\linewidth}
\begin{tabular}{lcc}
\toprule
\textbf{Model} & \textbf{Class} & \textbf{FID $\downarrow$} \\
\midrule
ADM-G~\cite{dhariwal2021diffusion} & Diff/FM & 2.97 \\
CDM~\cite{10.5555/3586589.3586636} & Diff/FM & 3.52 \\
Simple Diffusion~\cite{10.5555/3618408.3618945} & Diff/FM & 1.94 \\
RIN~\cite{10.5555/3618408.3619002} & Diff/FM & 2.75 \\
\midrule
FM + noise + denoising & FM & 6.8 \\
\textbf{TPC-FM + noise + denoising} & \textbf{FM} & \textbf{4.9} \\
\bottomrule
\end{tabular}
\end{adjustbox}
\end{minipage}
\hfill
\begin{minipage}[t]{0.49\linewidth}
\centering
\textbf{(b) ImageNet $64{\times}64$}\\[0.25em]
\begin{adjustbox}{max width=\linewidth}
\begin{tabular}{lcc}
\toprule
\textbf{Model} & \textbf{Class} & \textbf{FID $\downarrow$} \\
\midrule
EDM~\cite{NEURIPS2022_a98846e9} & Diff/FM & 1.55 \\
iDDPM~\cite{nichol2021improved} & Diff/FM & 2.92 \\
ADM (dropout)~\cite{dhariwal2021diffusion} & Diff/FM & 2.09 \\
\midrule
CD (LPIPS)~\cite{10.5555/3618408.3619743} & CM & 4.70 \\
iCT-deep~\cite{song2024improved} & CM & 3.25 \\
\midrule
FM + noise + denoising & FM & 3.6 \\
\textbf{TPC-FM + noise + denoising} & \textbf{FM} & \textbf{2.4} \\
\bottomrule
\end{tabular}
\end{adjustbox}
\end{minipage}

\vspace{-0.25em}
\end{table}
\paragraph{Higher-resolution and modern-backbone extensions.}
To test whether the same training-time modification transfers beyond the original U-Net FM setting, we evaluate TPC on DiT-XL/2 at $256{\times}256$ and $512{\times}512$, and on MeanFlow at $256{\times}256$, under matched inference settings. Table~\ref{tab:tpc_ablation_combined}(c) reports consistent FID improvements: DiT-XL/2 improves from 2.27 to 2.01 at $256{\times}256$ and from 3.04 to 2.56 at $512{\times}512$, while MeanFlow improves from 3.43 to 2.87 at 1 NFE. Since each baseline and TPC-augmented variant uses the same NFE, these results support improved sample quality at matched inference cost rather than reduced sampling cost.

\paragraph{Qualitative comparisons.}
In addition to the quantitative results, the appendix collects representative unconditional samples on CIFAR-10, ImageNet-32, ImageNet-64, and ImageNet-128 (Figures~\ref{fig:qual_samples} and~\ref{fig:imagenet128_qual}), class-conditional ImageNet-256 generations (Figure~\ref{fig:imagenet256_qual}), and multi-page class-conditional ImageNet-512 side-by-side DiT/TPC-DiT comparisons (Figure~\ref{fig:qualitative_samples_appendix}). Across these figures, the TPC-augmented models produce samples that are visually sharper, more coherent, and more class-consistent, especially in the higher-resolution conditional settings. The qualitative evidence is therefore consistent with the matched-FID improvements reported above and suggests that the gains reflect perceptible improvements in sample quality rather than only numerical changes in the evaluation metric.

\vspace{-1em}
\subsection{Ablation Studies}
\label{sec:ablations}

Table~\ref{tab:tpc_ablation_combined} examines the temporal-pairing probability
$p_{\mathrm{tpc}}$ and TPC coupling strength $\lambda_{\mathrm{tpc}}$ for fixed and
learned pairing on CIFAR-10. Across a broad range of settings, TPC-FM improves over the
FM w/ OT baseline FID of 6.35, but the gains are not monotonic in coupling strength. Moderate coupling 
gives the best sample quality, while overly strong coupling can over-regularize the vector field and 
degrade FID, consistent with the diagnostic pattern in Table~\ref{tab:rf_mechanism_combined}(c). The learned 
pairing variant gives the best overall result at $p_{\mathrm{tpc}}=0.75$ and 
$\lambda_{\mathrm{tpc}}=0.10$ with FID 3.19, improving over the best fixed-pairing setting of 3.59. Appendix Table~\ref{tab:learned_pair_noncollapse} further verifies that the best learned pairing does not collapse to the identity map.

\begin{table}[!htbp]
\centering
\caption{
\textbf{CIFAR-10 ablations and modern-backbone extensions for TPC-FM.}
Panel (a) reports the fixed-pairing ablation, and panel (b) reports the learned-pairing
ablation across temporal-pairing probability $p_{\mathrm{tpc}}$ and coupling strength
$\lambda_{\mathrm{tpc}}$. Panel (c) reports modern-backbone extensions under matched
inference settings.
}
\label{tab:tpc_ablation_combined}
\scriptsize
\setlength{\tabcolsep}{2.0pt}

\begin{minipage}[t]{0.32\linewidth}
\centering
\textbf{(a) Fixed pairing}\\[0.25em]
\resizebox{\linewidth}{!}{
\begin{tabular}{c|cccc|c}
\toprule
$p_{\mathrm{tpc}}$ & \multicolumn{4}{c|}{$\lambda_{\mathrm{tpc}}$} & Best \\
\cmidrule(lr){2-5}
 & 0.1 & 0.25 & 0.5 & 1.0 & \\
\midrule
0.25 & 3.895 & 3.833 & 4.104 & 4.314 & 3.833 \\
0.50 & 5.101 & 4.794 & 4.756 & 4.216 & 4.216 \\
0.75 & 4.089 & 4.101 & 4.075 & 3.931 & 3.931 \\
1.00 & 4.309 & \textbf{3.594} & 4.179 & 6.545 & \textbf{3.594} \\
\bottomrule
\end{tabular}
}
\end{minipage}
\hfill
\begin{minipage}[t]{0.32\linewidth}
\centering
\textbf{(b) Learned pairing}\\[0.25em]
\resizebox{\linewidth}{!}{
\begin{tabular}{c|cccc|c}
\toprule
$p_{\mathrm{tpc}}$ & \multicolumn{4}{c|}{$\lambda_{\mathrm{tpc}}$} & Best \\
\cmidrule(lr){2-5}
 & 0.1 & 0.25 & 0.5 & 1.0 & \\
\midrule
0.25 & 3.922 & 4.249 & 4.767 & 4.513 & 3.922 \\
0.50 & 3.895 & 3.885 & 4.266 & 4.723 & 3.885 \\
0.75 & \textbf{3.193} & 3.620 & 4.016 & 3.898 & \textbf{3.193} \\
1.00 & 4.402 & 4.382 & 4.914 & 6.223 & 4.382 \\
\bottomrule
\end{tabular}
}
\end{minipage}
\hfill
\begin{minipage}[t]{0.32\linewidth}
\centering
\textbf{(c) Modern backbones}\\[0.25em]
\resizebox{\linewidth}{!}{
\begin{tabular}{lcccc}
\toprule
Family & Base FID & \textbf{TPC FID} & $\Delta$ FID & NFE \\
\midrule
DiT-XL/2 @ 256 & 2.27 & \textbf{2.01} & 0.26 & 250 \\
DiT-XL/2 @ 512 & 3.04 & \textbf{2.56} & 0.48 & 250 \\
MeanFlow @ 256 & 3.43 & \textbf{2.87} & 0.56 & 1 \\
\bottomrule
\end{tabular}
}
\end{minipage}

\vspace{-0.25em}
\end{table}

\vspace{-1em}
\FloatBarrier
\section{Conclusion}
\vspace{-1em}
We introduced Temporal Pair Consistency (TPC), a simple shared-path training objective for flow-matching-style generative models. TPC couples velocity predictions at paired timesteps along the same probability path while leaving the model architecture, probability path, sampler, and inference-time computation unchanged. Across flow matching, OT-CFM, rectified flow, DiT-style backbones, and MeanFlow, TPC improves sample quality under matched inference settings. Mechanism diagnostics further show that these gains are accompanied by lower measured gradient variance, higher paired-gradient correlation, and reduced temporal roughness, while random and local pairing controls are substantially weaker. Overall, our results support TPC as a lightweight and modular training-time modification for improving continuous-time generative models at matched inference cost.

\section*{Limitations}
Broader evaluation on text-to-image generation and other modalities remains future work.

\begin{ack}
This work was partially supported by the Commonwealth Cyber Initiative (CCI) program (H-2Q25-020), the William \& Mary Faculty Research Award, the Modal Academic Compute Award, and computational resources provided by the NSF ACCESS program with support from the NCSA Delta and DeltaAI systems.
\end{ack}

\section*{Competing Interests}
The authors declare that they have no known competing interests.

% \begingroup\sloppy\input{references.bbl}\endgroup
\bibliographystyle{plainnat}
\bibliography{reference}

\clearpage
\section*{Appendix}
\label{sec:appendix}

\section*{Appendix Contents}
\label{app:contents}
\vspace{-0.25em}
\begin{center}
\small
\setlength{\tabcolsep}{0pt}
\renewcommand{\arraystretch}{1.12}
\begin{tabular*}{0.98\textwidth}{@{\extracolsep{\fill}}p{0.84\textwidth}r@{}}
\toprule
\textbf{Appendix entry} & \textbf{Page} \\
\midrule
\textbf{\ref{app:experimental_protocol}. Supplementary Experimental Protocol} & \pageref{app:experimental_protocol} \\
\quad \ref{app:matched_cost}. Matched-cost evaluation & \pageref{app:matched_cost} \\
\quad\quad Table~\ref{tab:matched_compute_controls}: training-overhead and equal-compute controls & \pageref{tab:matched_compute_controls} \\
\quad \ref{app:tpc_family_implementation}. Implementation across model families & \pageref{app:tpc_family_implementation} \\
\quad\quad Table~\ref{tab:tpc_family_implementation}: applying TPC across model families & \pageref{tab:tpc_family_implementation} \\
\quad\quad Algorithm~\ref{alg:tpc_family_wrapper}: family-agnostic TPC wrapper & \pageref{alg:tpc_family_wrapper} \\
\quad \ref{app:compute_hardware}. Compute hardware & \pageref{app:compute_hardware} \\

\quad \ref{app:tpc_algorithm}. Training algorithm & \pageref{app:tpc_algorithm} \\
\quad\quad Algorithm~\ref{alg:tpc_all}: Flow Matching with Temporal Pair Consistency & \pageref{alg:tpc_all} \\
\quad \ref{app:pairing_controls}. TPC variants and pairing controls & \pageref{app:pairing_controls} \\
\quad \ref{app:diagnostics}. Diagnostic measurements & \pageref{app:diagnostics} \\
\addlinespace[0.25em]
\textbf{\ref{app:supp_quant}. Supplementary Quantitative Results} & \pageref{app:supp_quant} \\
\quad \ref{app:consistency_reference}. Consistency-FM composition experiment & \pageref{app:consistency_reference} \\
\quad\quad Table~\ref{tab:consistency_reference}: matched 2-NFE Consistency-FM/TPC comparison & \pageref{tab:consistency_reference} \\
\quad \ref{app:learned_pair_noncollapse}. Learned-pairing non-collapse diagnostic & \pageref{app:learned_pair_noncollapse} \\
\quad\quad Table~\ref{tab:learned_pair_noncollapse}: learned-pairing non-collapse diagnostic & \pageref{tab:learned_pair_noncollapse} \\
\quad \ref{app:full_rf_results}. Full rectified-flow results & \pageref{app:full_rf_results} \\
\quad\quad Table~\ref{tab:rf_benchmark_full_appendix}: full RF one-step and full-simulation results & \pageref{tab:rf_benchmark_full_appendix} \\
\quad \ref{app:training_phase_diagnostics_section}. Training-phase diagnostics & \pageref{app:training_phase_diagnostics_section} \\
\quad\quad Table~\ref{tab:training_phase_diagnostics}: early/mid/late diagnostic trends & \pageref{tab:training_phase_diagnostics} \\
\quad \ref{app:quant_figures}. Training-dynamics and matched-cost visualizations & \pageref{app:quant_figures} \\
\quad\quad Figure~\ref{fig:variance_adv}: training diagnostic behavior & \pageref{fig:variance_adv} \\
\quad\quad Figure~\ref{fig:teaser}: matched-cost teaser visualization & \pageref{fig:teaser} \\
\addlinespace[0.25em]
\textbf{\ref{app:supp_qual}. Supplementary Qualitative Results} & \pageref{app:supp_qual} \\
\quad Figure~\ref{fig:qual_samples}: CIFAR-10, ImageNet-32, and ImageNet-64 samples & \pageref{fig:qual_samples} \\
\quad Figure~\ref{fig:imagenet128_qual}: ImageNet-128 samples & \pageref{fig:imagenet128_qual} \\
\quad Figure~\ref{fig:imagenet256_qual}: ImageNet-256 class-conditional samples & \pageref{fig:imagenet256_qual} \\
\quad Figure~\ref{fig:qualitative_samples_appendix}: ImageNet-512 DiT/TPC-DiT comparisons & \pageref{fig:qualitative_samples_appendix}--\pageref{fig:qualitative_samples_appendix_p5} \\
\addlinespace[0.25em]
\textbf{\ref{app:impact_statement}. Impact Statement} & \pageref{app:impact_statement} \\
\bottomrule
\end{tabular*}
\end{center}
\vspace{0.25em}

\section{Supplementary Experimental Protocol}
\label{app:experimental_protocol}

This appendix contains supplementary protocol details, quantitative results, and qualitative samples. The main empirical claim of the paper is matched-cost improvement: TPC changes the training objective by adding a shared-path paired velocity-consistency term, while leaving the model architecture, probability path, sampler, and inference-time computation unchanged.

\subsection{Matched-cost evaluation}
\label{app:matched_cost}

For the primary comparisons, each TPC-augmented model is evaluated against the corresponding baseline using the same backbone, data preprocessing, probability path, solver family, and number of function evaluations (NFE). We therefore report NFE as an accounting variable for inference cost, not as a claim that TPC universally reduces the number of solver steps. This protocol is intended to isolate the effect of the training-time paired-consistency objective.

The cleanest matched comparisons are the FM-family and RF-family comparisons in the main text. For example, TPC-FM is compared to FM with the same NFE, TPC+OT-CFM is compared to OT-CFM with the same NFE, and TPC-RF is compared to the corresponding RF baseline under the same one-step or full-simulation setting. The higher-resolution DiT-style and MeanFlow extensions follow the same principle: baseline and TPC-augmented variants use matched inference settings.

\begin{table}[H]
\centering
\small
\caption{Equal-compute CIFAR-10 controls. Measurements use matched hardware, precision, batch size, and data pipelines. Relative steps, wall time, and peak memory are normalized to FM. FID values are shown as reported.}
\label{tab:matched_compute_controls}
\begin{tabularx}{\linewidth}{@{}Xcccc@{}}
\toprule
\textbf{Method} & \makecell{\textbf{Relative}\\\textbf{steps}} & \makecell{\textbf{Relative}\\\textbf{wall time}} & \makecell{\textbf{Peak}\\\textbf{memory}} & \textbf{FID$\downarrow$} \\
\midrule
FM & $1.00\times$ & $1.00\times$ & $1.00\times$ & $6.35 \pm 0.11$ \\
FM, longer training & $1.70\times$ & $1.70\times$ & $1.00\times$ & $5.48 \pm 0.09$ \\
FM + second-timestep supervision, no consistency & $1.00\times$ & $1.68\times$ & $1.04\times$ & $5.72 \pm 0.10$ \\
FM + random pairing & $1.00\times$ & $1.68\times$ & $1.04\times$ & $5.95 \pm 0.12$ \\
FM + local pairing & $1.00\times$ & $1.68\times$ & $1.04\times$ & $5.20 \pm 0.09$ \\
Fixed TPC & $1.00\times$ & $1.68\times$ & $1.04\times$ & $3.59 \pm 0.06$ \\
Learned TPC & $1.00\times$ & $1.70\times$ & $1.05\times$ & $3.19 \pm 0.05$ \\
\bottomrule
\end{tabularx}
\end{table}

\subsection{Implementation Details for Applying TPC Across Model Families}
\label{app:tpc_family_implementation}

TPC is implemented as a training-time paired-evaluation wrapper around an existing continuous-time generative objective. It does not change the model architecture, probability path, sampler, solver tolerance, number of function evaluations, or inference-time computation. For every baseline family, the implementation consists of four steps: (i) sample the same training tuple used by the original method; (ii) evaluate the model at the original timestep $t$; (iii) reuse the same endpoint, noise, or path randomness to evaluate an auxiliary timestep $t'=\psi(t)$; and (iv) add a consistency penalty between the two predictions after expressing them in the same velocity or velocity-equivalent parameterization.

Given a path sampler $\mathcal{P}$ and endpoint randomness $(x_0,x_1)$, we write
\[
(x_t,u_t)=\mathcal{P}(x_0,x_1,t),
\qquad
(x_{t'},u_{t'})=\mathcal{P}(x_0,x_1,t'),
\qquad
 t'=\psi(t).
\]
The baseline supervision is unchanged, $\mathcal{L}_{\mathrm{base}}=\|v_\theta(x_t,t)-u_t\|_2^2$. TPC adds a gated paired term,
\[
\mathcal{L}_{\mathrm{TPC}}
=
b\,\lambda_{\mathrm{TPC}}
\|v_\theta(x_t,t)-v_\theta(x_{t'},t')\|_2^2,
\qquad
b\sim\mathrm{Bernoulli}(p_{\mathrm{TPC}}),
\]
so that the full objective is
\[
\mathcal{L}=\mathcal{L}_{\mathrm{base}}+\mathcal{L}_{\mathrm{TPC}}+\lambda_{\mathrm{mono}}r_{\mathrm{mono}}(\psi).
\]
The monotonicity term is used only for the learned pairing rule. Unless otherwise stated, we use the learned monotone pairing with $p_{\mathrm{TPC}}=0.75$, $\lambda_{\mathrm{TPC}}=0.10$, $\lambda_{\mathrm{mono}}=10^{-3}$, and $K=32$ grid points for the monotonicity penalty. The fixed-pairing variant uses $\psi(t)=1-t$ and sets $\lambda_{\mathrm{mono}}=0$.

\paragraph{Flow matching and OT-CFM.}
For standard FM and OT-CFM, TPC directly reuses the original conditional path sampler. The paired point $(x_{t'},u_{t'})$ is generated from the same endpoint pair or minibatch-OT assignment as $(x_t,u_t)$. The primary supervision remains the usual velocity regression at $t$; the auxiliary evaluation at $t'$ contributes only through the TPC consistency term. No additional target supervision is applied at $t'$ unless it is already part of the baseline method.

\paragraph{Rectified flow.}
For rectified flow, the same rectified interpolation path and target velocity used by the baseline are retained. TPC samples the paired timestep $t'=\psi(t)$ along the same rectified path and applies the consistency term between $v_\theta(x_t,t)$ and $v_\theta(x_{t'},t')$. The rectification stage, sampler, solver, tolerance, and NFE are unchanged.

\paragraph{DiT-style velocity models.}
For DiT backbones trained with a native velocity-prediction objective, including the DiT-XL/2 extensions reported in the paper, TPC is applied directly to the predicted velocity outputs at $t$ and $t'$. Thus both evaluations use the same DiT architecture, patch size, class/text conditioning, noise schedule, solver, and NFE as the baseline; only the training loss receives the additional paired term.

For DiT implementations that predict $\epsilon$ or $x_0$ rather than velocity, we first convert both predictions to the velocity-equivalent target associated with the baseline interpolation before computing the TPC penalty. For a diffusion interpolation $x_t=\alpha_t x_0+\sigma_t\epsilon$, the velocity target is $u_t=\dot{\alpha}_t x_0+\dot{\sigma}_t\epsilon$. Therefore an $x_0$-prediction $\hat{x}_0$ is converted using $\hat{\epsilon}=(x_t-\alpha_t\hat{x}_0)/\sigma_t$ and $\hat{u}_t=\dot{\alpha}_t\hat{x}_0+\dot{\sigma}_t\hat{\epsilon}$; an $\epsilon$-prediction $\hat{\epsilon}$ is converted using $\hat{x}_0=(x_t-\sigma_t\hat{\epsilon})/\alpha_t$ and the same formula for $\hat{u}_t$. The paired loss is always computed in this common vector-field space, not directly between heterogeneous parameterizations.

\paragraph{Noise-augmented ImageNet flow matching.}
For noise-augmented FM experiments, the same additive noise protocol and denoising rule are used for both baseline and TPC models. The paired point is generated with the same endpoint randomness and the same noise-augmentation protocol as the original point. The noise scale $\sigma$ is selected under the baseline setting and then held fixed for TPC, so the comparison isolates the effect of temporal pairing rather than noise-hyperparameter tuning.

\paragraph{MeanFlow.}
For MeanFlow-style average-velocity matching, TPC is applied to the model's average-velocity prediction in the native MeanFlow parameterization. The paired timestep uses the same data/noise randomness and the same sampling schedule as the baseline MeanFlow run. The one-step sampler and NFE remain unchanged; TPC modifies only the training objective by requiring the paired average-velocity predictions to be consistent.

\paragraph{What is held fixed.}
Across all families, the model backbone, path sampler, endpoint pairing, data preprocessing, optimizer, training steps, batch size, sampler, solver tolerance, and NFE are held fixed between the baseline and the TPC variant. TPC adds one extra model evaluation only on minibatches where $b=1$. At inference time, the learned vector field is sampled exactly as in the corresponding baseline.

\begin{table*}[!t]
\centering
\footnotesize
\setlength{\tabcolsep}{4.5pt}
\renewcommand{\arraystretch}{1.05}
\caption{
\textbf{Applying TPC across model families.}
TPC is always computed between two predictions expressed in the same velocity or velocity-equivalent space and evaluated along the same probability path/randomness. The sampler, solver, architecture, and NFE are unchanged for every family.
}
\label{tab:tpc_family_implementation}
\resizebox{\textwidth}{!}{%
\begin{tabular}{@{}l l l l@{}}
\toprule
Family & Paired quantity & Same-path randomness & Inference change? \\
\midrule
FM & velocity $v_\theta(x_t,t)$ & same endpoint pair $(x_0,x_1)$ & none \\
OT-CFM & velocity $v_\theta(x_t,t)$ & same minibatch-OT pair / endpoint assignment & none \\
Rectified Flow & rectified-flow velocity & same rectified path endpoints & none \\
DiT native velocity model & native velocity prediction & same noise/data pair and timestep schedule & none \\
DiT $\epsilon$/$x_0$ model & converted velocity-equivalent prediction & same noise/data pair and timestep schedule & none \\
Noise-augmented FM & velocity under fixed noise protocol & same endpoints and fixed noise augmentation & none \\
MeanFlow & native average-velocity prediction & same noise/data randomness & none \\
\bottomrule
\end{tabular}%
}
\end{table*}

\begin{algorithm}[H]
\caption{Family-agnostic TPC wrapper}
\label{alg:tpc_family_wrapper}
\begin{algorithmic}[1]
\REQUIRE Baseline path sampler $\mathcal{P}$, model prediction map $v_\theta$, pairing rule $\psi$, $p_{\mathrm{TPC}}$, $\lambda_{\mathrm{TPC}}$.
\STATE Sample baseline training randomness $(x_0,x_1)$ and timestep $t$.
\STATE Compute baseline point and target: $(x_t,u_t)=\mathcal{P}(x_0,x_1,t)$.
\STATE Predict primary velocity: $v_t=v_\theta(x_t,t)$.
\STATE Compute baseline loss: $\mathcal{L}_{\mathrm{base}}=\|v_t-u_t\|_2^2$.
\STATE Sample gate $b\sim\mathrm{Bernoulli}(p_{\mathrm{TPC}})$.
\IF{$b=1$}
    \STATE Set $t'=\psi(t)$.
    \STATE Reuse the same path randomness: $(x_{t'},u_{t'})=\mathcal{P}(x_0,x_1,t')$.
    \STATE Predict paired velocity: $v_{t'}=v_\theta(x_{t'},t')$.
    \STATE Compute $\mathcal{L}_{\mathrm{TPC}}=\lambda_{\mathrm{TPC}}\|v_t-v_{t'}\|_2^2$.
\ELSE
    \STATE Set $\mathcal{L}_{\mathrm{TPC}}=0$.
\ENDIF
\STATE Add monotonicity penalty only if $\psi$ is learned.
\STATE Update using $\mathcal{L}_{\mathrm{base}}+\mathcal{L}_{\mathrm{TPC}}+\lambda_{\mathrm{mono}}r_{\mathrm{mono}}(\psi)$.
\end{algorithmic}
\end{algorithm}

\subsection{Compute Hardware}
\label{app:compute_hardware}

All reported experiments were run on an internal GPU cluster using NVIDIA H200 GPUs with 141GB HBM3e memory per GPU and local high-throughput training storage. Lower-resolution CIFAR-10 and ImageNet runs used fewer GPUs, while the largest reported high-resolution run, DiT-XL/2 on ImageNet-$512{\times}512$, used 32 NVIDIA H200 GPUs. Our main comparisons report NFE to account for matched inference cost; this subsection documents the compute environment used for the reported experiments. The project also included smaller pilot and tuning runs, but the paper reports the final matched experimental runs.

\subsection{Training algorithm}
\label{app:tpc_algorithm}

\begin{algorithm}[H]
\caption{Flow Matching with Temporal Pair Consistency (TPC)}
\label{alg:tpc_all}
\begin{algorithmic}[1]

\REQUIRE Model $f_\theta$, path sampler $\mathcal{P}$, pairing rule $\psi$, TPC probability $p_{\mathrm{tpc}}$, weights $\lambda_{\mathrm{tpc}}$ and $\lambda_{\mathrm{mono}}$.

\STATE \textbf{Fixed pairing:} $\psi(t)=1-t$.
\STATE \textbf{Learned pairing:} $\psi(t)=\phi(t)$, with $\phi$ defined by Eq.~\eqref{eq:learned_pairing}.

\FOR{each minibatch $(x_1)$}
    \STATE Sample $x_0$ and $t \sim \mathcal{U}(0,1)$.
    \STATE $(x_t, u_t) = \mathcal{P}(x_0, x_1, t)$.
    \STATE $v_t = f_\theta(x_t, t)$.
    \STATE $\mathcal{L}_{\mathrm{FM}} = \| v_t - u_t \|_2^2$.

    \STATE Sample $b \sim \mathrm{Bernoulli}(p_{\mathrm{tpc}})$.
    \IF{$b = 1$}
        \STATE $t' = \psi(t)$.
        \STATE $(x_{t'}, u_{t'}) = \mathcal{P}(x_0, x_1, t')$.
        \STATE $v_{t'} = f_\theta(x_{t'}, t')$.
        \STATE $\mathcal{L}_{\mathrm{TPC}} = \| v_t - v_{t'} \|_2^2$.
    \ELSE
        \STATE $\mathcal{L}_{\mathrm{TPC}} = 0$.
    \ENDIF

    \IF{using learned pairing}
        \STATE Construct grid $g_1,\ldots,g_K$.
        \STATE $r_{\mathrm{mono}}(\phi)=\sum_{i=1}^{K-1}\mathrm{softplus}(\phi(g_i)-\phi(g_{i+1}))$.
    \ELSE
        \STATE $r_{\mathrm{mono}}(\phi)=0$.
    \ENDIF

    \STATE $\mathcal{L}=\mathcal{L}_{\mathrm{FM}}
    +\lambda_{\mathrm{tpc}}\mathcal{L}_{\mathrm{TPC}}
    +\lambda_{\mathrm{mono}}r_{\mathrm{mono}}(\phi)$.
    \STATE Update trainable parameters of $f_\theta$ and, when learned pairing is used, $\phi$.
\ENDFOR

\STATE \textbf{return} $f_\theta$.

\end{algorithmic}
\end{algorithm}

\subsection{TPC variants and pairing controls}
\label{app:pairing_controls}

We evaluate TPC using both fixed and learned pairing rules. The fixed rule uses the parameter-free symmetric map $\psi_{\mathrm{fix}}(t)=1-t$, which adds no learned parameters and cannot collapse to self-pairing. The learned variant uses a low-capacity monotone map $\phi(t)$, as described in the method section, and gives the strongest CIFAR-10 result in our ablations. We therefore report both variants: fixed pairing serves as a non-collapsible structured reference, while learned pairing tests whether adaptive temporal pairing can further improve the TPC objective.

To test whether the gains come from the specific shared-path pairing structure rather than arbitrary temporal coupling, we include two controls. Random pairing samples an independent auxiliary timestep $t'\sim U[0,1]$. Local pairing samples an auxiliary timestep near the primary timestep $t$. These controls use the same training setup as TPC and are reported in the main diagnostic table. Their weaker FID and diagnostic improvements support the interpretation that the structured pairing choice matters empirically.

\subsection{Diagnostic measurements}
\label{app:diagnostics}

The diagnostic tables in the main text report empirical measurements of the training dynamics under matched CIFAR-10 settings. Gradient variance is reported as a ratio to the matched FM baseline. Paired-gradient correlation measures alignment between the two paired training evaluations. Temporal roughness measures the finite-difference variation of velocity predictions along paired path states and is also normalized by the FM baseline. These quantities are diagnostic measurements only: they are used to characterize how TPC changes training behavior, not as formal guarantees.

The primary empirical interpretation is improved training-time vector-field estimation, supported by gradient-variance, paired-gradient-correlation, temporal-roughness, and pairing-control measurements. Excessive coupling may introduce fitting bias by forcing evaluations of genuinely different marginal fields toward agreement. These diagnostics alone do not establish a quantitative bias–variance decomposition or predict the coupling strength that minimizes FID. A local Taylor/ODE truncation argument applies to nearby temporal pairs under appropriate smoothness assumptions; it does not establish a global result for the nonlocal fixed rule $\psi(t)=1-t$. Accordingly, high-NFE improvements under adaptive ODE solvers with $10^{-5}$ tolerance are interpreted primarily as training-time gains, not as evidence of reduced numerical discretization error.

\section{Supplementary Quantitative Results}
\label{app:supp_quant}

This section contains additional tabular and diagnostic results that complement the main matched-cost comparisons.

\subsection{Consistency-FM composition experiment}
\label{app:consistency_reference}

We compare Consistency-FM with fixed- and learned-TPC compositions in the same codebase on CIFAR-10. All runs use the same backbone, dataset, optimizer, evaluation protocol, and 2-NFE sampler; relative training wall time is normalized to the reproduced Consistency-FM run. Table~\ref{tab:consistency_reference} reports mean $\pm$ standard deviation over five independent training seeds. The Consistency-FM result of $5.34 \pm 0.10$ is our independent reproduction; its mean coincides with the published FID of $5.34$.

\begin{table}[H]
\centering
\caption{Same-codebase CIFAR-10 Consistency-FM composition comparison. All variants use the same backbone, dataset, optimizer, evaluation protocol, and 2-NFE sampler. Relative training wall time is normalized to the Consistency-FM reproduction. FID values are mean $\pm$ standard deviation over five independent training seeds; the reproduced baseline mean coincides with the published FID of $5.34$.}
\label{tab:consistency_reference}
\small
\setlength{\tabcolsep}{4pt}
\begin{tabular}{lccc}
\toprule
\textbf{Method} & \textbf{NFE$\downarrow$} & \makecell[c]{\textbf{Relative training}\\\textbf{wall time}} & \textbf{FID$\downarrow$} \\
\midrule
Consistency-FM & 2 & $1.00\times$ & $5.34 \pm 0.10$ \\
Consistency-FM, longer training & 2 & $1.70\times$ & $5.08 \pm 0.09$ \\
Consistency-FM + fixed TPC & 2 & $1.68\times$ & $4.70 \pm 0.08$ \\
\textbf{Consistency-FM + learned TPC} & \textbf{2} & $1.70\times$ & \textbf{$4.29 \pm 0.07$} \\
\bottomrule
\end{tabular}
\end{table}

\subsection{Learned-pairing non-collapse diagnostic}
\label{app:learned_pair_noncollapse}

Because the learned pairing map could in principle approach the identity map $\phi(t)=t$, we include a direct diagnostic of the final learned maps. We evaluate each pairing rule on a fixed grid $\{g_k\}_{k=1}^{K}$ and report the mean separation from identity,
\[
\Delta_{\mathrm{id}}
=
\frac{1}{K}\sum_{k=1}^{K} |\psi(g_k)-g_k|,
\]
the near-identity fraction,
\[
\pi_{\mathrm{id}}
=
\frac{1}{K}\sum_{k=1}^{K}
\mathbf{1}\{|\psi(g_k)-g_k|<0.05\},
\]
and the range of the learned auxiliary timesteps,
\[
\mathrm{Range}(\psi)
=
\max_k \psi(g_k)-\min_k \psi(g_k).
\]
A collapsed learned map would have small $\Delta_{\mathrm{id}}$ and large $\pi_{\mathrm{id}}$.

\begin{table}[H]
\centering
\caption{
Learned-pairing non-collapse diagnostic. The best learned TPC model maintains substantial separation from identity while achieving the lowest FID. In contrast, the overcoupled setting is much closer to identity and has poor FID, indicating that near-collapse is associated with a failed regime rather than the best-performing learned model.
}
\label{tab:learned_pair_noncollapse}
\small
\setlength{\tabcolsep}{4pt}
\resizebox{\linewidth}{!}{
\begin{tabular}{lcccc}
\toprule
\textbf{Variant} &
\textbf{FID $\downarrow$} &
\textbf{Mean $|\psi(t)-t|$ $\uparrow$} &
\textbf{Near-id. frac. $\downarrow$} &
\textbf{Range of $\psi(t)$ $\uparrow$} \\
\midrule
Identity control $\psi(t)=t$ & 6.35 & 0.00 & 1.00 & 1.00 \\
Fixed symmetric $\psi(t)=1-t$ & 3.59 & 0.50 & 0.05 & 1.00 \\
\textbf{Learned TPC, best} & \textbf{3.19} & \textbf{0.24} & \textbf{0.04} & \textbf{0.83} \\
Learned TPC, moderate & 3.88 & 0.20 & 0.12 & 0.74 \\
Learned TPC, overcoupled & 6.22 & 0.06 & 0.75 & 0.45 \\
\bottomrule
\end{tabular}
}
\end{table}

Table~\ref{tab:learned_pair_noncollapse} shows that the best learned TPC run does not collapse to self-pairing: it has mean separation $\Delta_{\mathrm{id}}=0.24$, near-identity fraction $\pi_{\mathrm{id}}=0.04$, and auxiliary-timestep range $0.83$. By contrast, the overcoupled setting is much closer to identity, with mean separation $0.06$ and near-identity fraction $0.75$, and its FID degrades to $6.22$. Thus, identity collapse is not the source of the best learned-pairing result; when near-collapse appears, it corresponds to a poor-performing overcoupled regime.

\subsection{Full rectified-flow results}
\label{app:full_rf_results}

\begin{table}[H]
\centering
\setlength{\tabcolsep}{4pt}
\caption{
Full rectified-flow and probability-flow comparisons. 
The main paper reports the adaptive-RK45 full-simulation subset; this table includes both one-step/distillation and full-simulation regimes.
}
\label{tab:rf_benchmark_full_appendix}
\resizebox{0.62\textwidth}{!}{
\begin{tabular}{lcccc}
\toprule
\textbf{Method} &
\textbf{NFE ($\downarrow$)} &
\textbf{IS ($\uparrow$)} &
\textbf{FID ($\downarrow$)} &
\textbf{Recall ($\uparrow$)} \\
\midrule
\multicolumn{5}{c}{\textit{One-Step Generation + Distil (Euler solver, N=1)}} \\
\midrule
1-Rectified Flow & 1 & 9.08 & 6.18 & 0.45 \\
2-Rectified Flow & 1 & 9.01 & 4.85 & 0.50 \\
3-Rectified Flow & 1 & 8.79 & 5.21 & 0.51 \\
\midrule
\textbf{TPC-1RF (ours)} & \textbf{1} & \textbf{9.21} & \textbf{5.86} & \textbf{0.47} \\
\textbf{TPC-2RF (ours)} & \textbf{1} & \textbf{9.14} & \textbf{4.55} & \textbf{0.53} \\
\textbf{TPC-3RF (ours)} & \textbf{1} & \textbf{8.92} & \textbf{4.83} & \textbf{0.53} \\
\midrule
\multicolumn{5}{c}{\textit{Full Simulation (Runge--Kutta RK45), Adaptive N}} \\
\midrule
1-Rectified Flow & 127 & 9.60 & 2.58 & 0.57 \\
2-Rectified Flow & 110 & 9.24 & 3.36 & 0.54 \\
3-Rectified Flow & 104 & 9.01 & 3.96 & 0.53 \\
\midrule
\textbf{TPC-1RF (ours)} & \textbf{127} & \textbf{9.78} & \textbf{2.15} & \textbf{0.60} \\
\textbf{TPC-2RF (ours)} & \textbf{110} & \textbf{9.42} & \textbf{2.95} & \textbf{0.58} \\
\textbf{TPC-3RF (ours)} & \textbf{104} & \textbf{9.18} & \textbf{3.45} & \textbf{0.55} \\
\bottomrule
\end{tabular}
}
\end{table}

\subsection{Training-phase diagnostics}
\label{app:training_phase_diagnostics_section}

\begin{table}[H]
\centering
\caption{Training-phase diagnostics for TPC-FM relative to FM on CIFAR-10. The variance and correlation effects strengthen over training.}
\label{tab:training_phase_diagnostics}
\resizebox{\linewidth}{!}{%
\begin{tabular}{lccc}
\toprule
\textbf{Training phase} & \textbf{TPC-FM / FM grad var$\downarrow$} & \textbf{Grad corr$\uparrow$} & \textbf{Mid-time var ratio $([0.3,0.7])\downarrow$} \\
\midrule
Early & $0.84 \pm 0.03$ & $0.34 \pm 0.03$ & $0.79 \pm 0.03$ \\
Mid & $0.71 \pm 0.02$ & $0.52 \pm 0.03$ & $0.66 \pm 0.02$ \\
Late & $0.66 \pm 0.02$ & $0.61 \pm 0.03$ & $0.60 \pm 0.02$ \\
\bottomrule
\end{tabular}}
\end{table}

\subsection{Training-dynamics and matched-cost visualizations}
\label{app:quant_figures}

\begin{figure}[H]
    \centering
    \includegraphics[width=0.78\columnwidth]{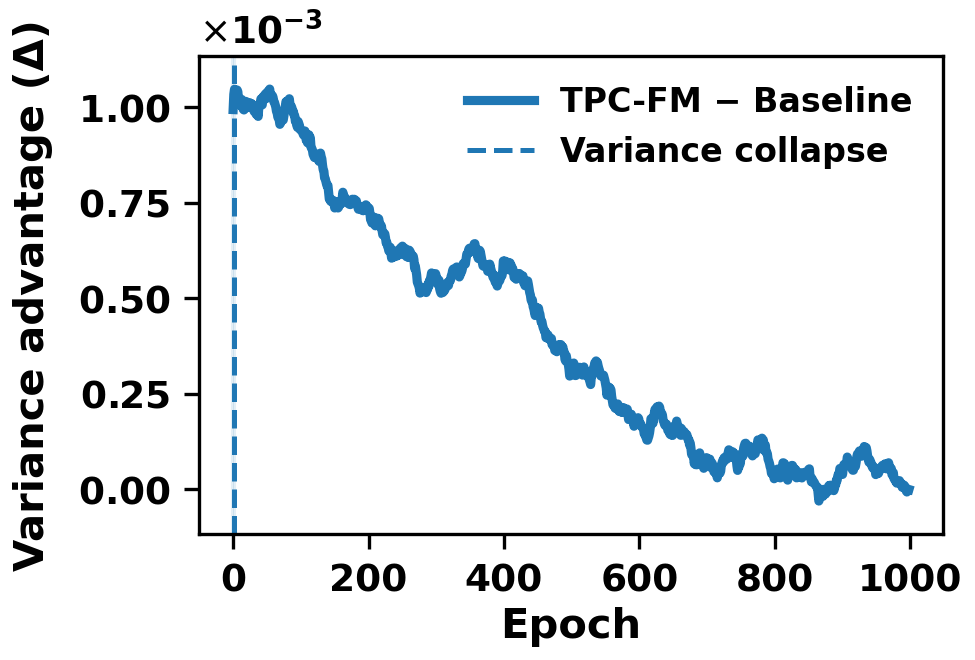}
    \caption{
    \textbf{Training diagnostic behavior.}
    TPC-FM shows lower measured gradient variance than the matched FM baseline throughout training.
    }
    \label{fig:variance_adv}
\end{figure}

\begin{figure}[H]
    \centering
    \includegraphics[width=\columnwidth]{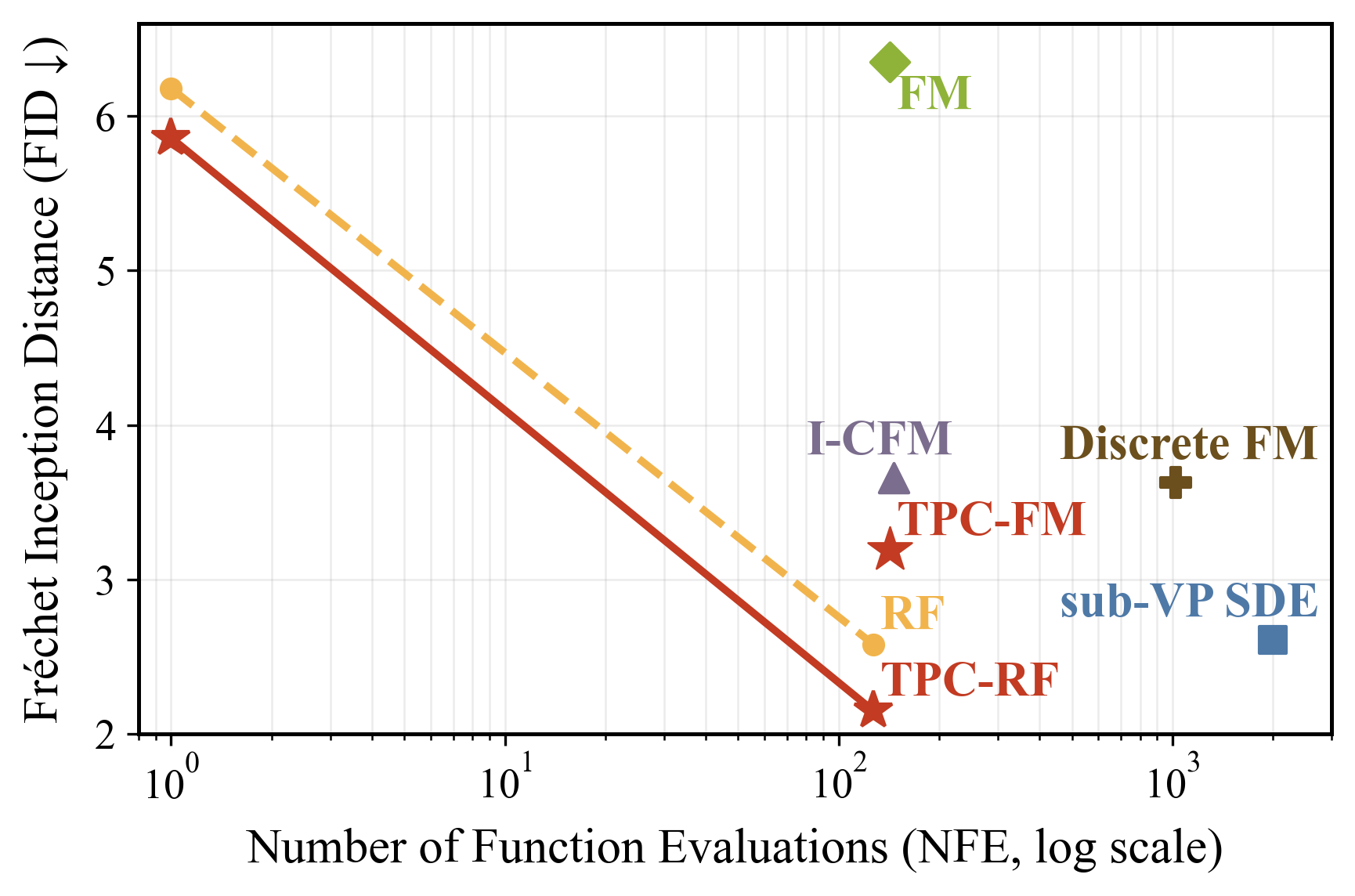}
    \caption{\textbf{Sample quality at matched inference cost}.
Fr\'echet Inception Distance (FID $\downarrow$) versus number of function evaluations (NFE, log scale) on CIFAR-10. Temporal Pair Consistency (TPC) is applied only during training and leaves the inference-time architecture, probability path, sampler, and NFE unchanged. Under the same sampling settings as the corresponding baselines, TPC variants achieve lower FID.
    }
    \label{fig:teaser}
\end{figure}

\section{Supplementary Qualitative Results}
\label{app:supp_qual}

This section collects representative qualitative samples for unconditional generation and higher-resolution class-conditional generation.

\begin{figure}[H]
    \centering
    \begin{subfigure}{0.32\columnwidth}
        \centering
        \includegraphics[width=\linewidth]{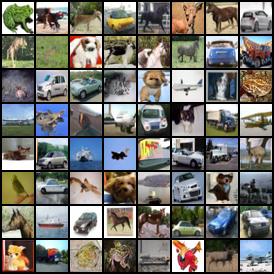}
        \caption{CIFAR-10}
        \label{fig:cifar10_qual}
    \end{subfigure}\hfill
    \begin{subfigure}{0.32\columnwidth}
        \centering
        \includegraphics[width=\linewidth]{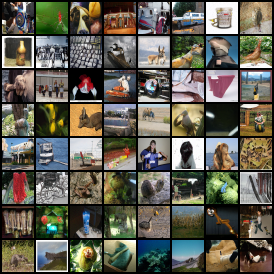}
        \caption{ImageNet-32}
        \label{fig:imagenet32_qual}
    \end{subfigure}\hfill
    \begin{subfigure}{0.32\columnwidth}
        \centering
        \includegraphics[width=\linewidth]{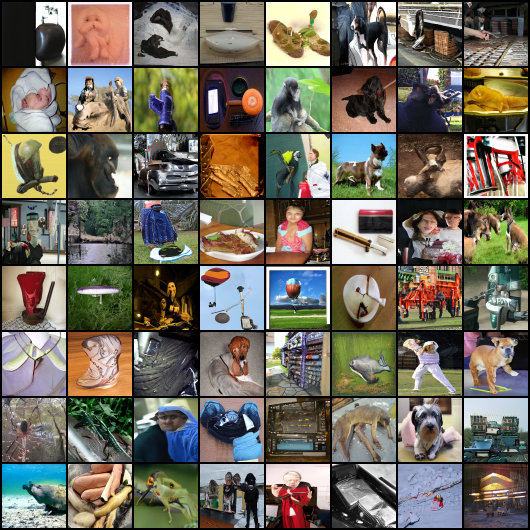}
        \caption{ImageNet-64}
        \label{fig:imagenet64_qual}
    \end{subfigure}

    \caption{
    \textbf{Qualitative samples.}
    Unconditional generations on CIFAR-10, ImageNet-32, and ImageNet-64.
    }
    \label{fig:qual_samples}
\end{figure}

\begin{figure}[H]
    \centering
    \includegraphics[width=0.64\columnwidth]{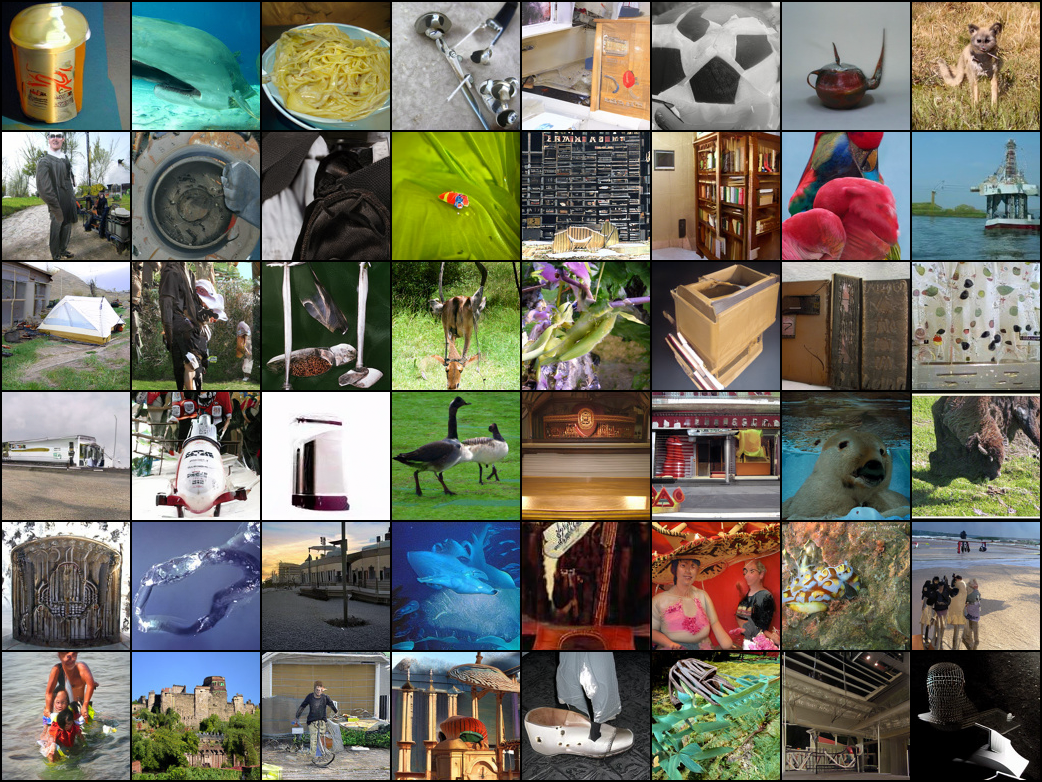}
    \caption{
    \textbf{ImageNet-128 qualitative samples.}
    Unconditional generations on ImageNet-128.
    }
    \label{fig:imagenet128_qual}
\end{figure}

\begin{figure}[H]
    \centering
    \includegraphics[page=1,width=\textwidth]{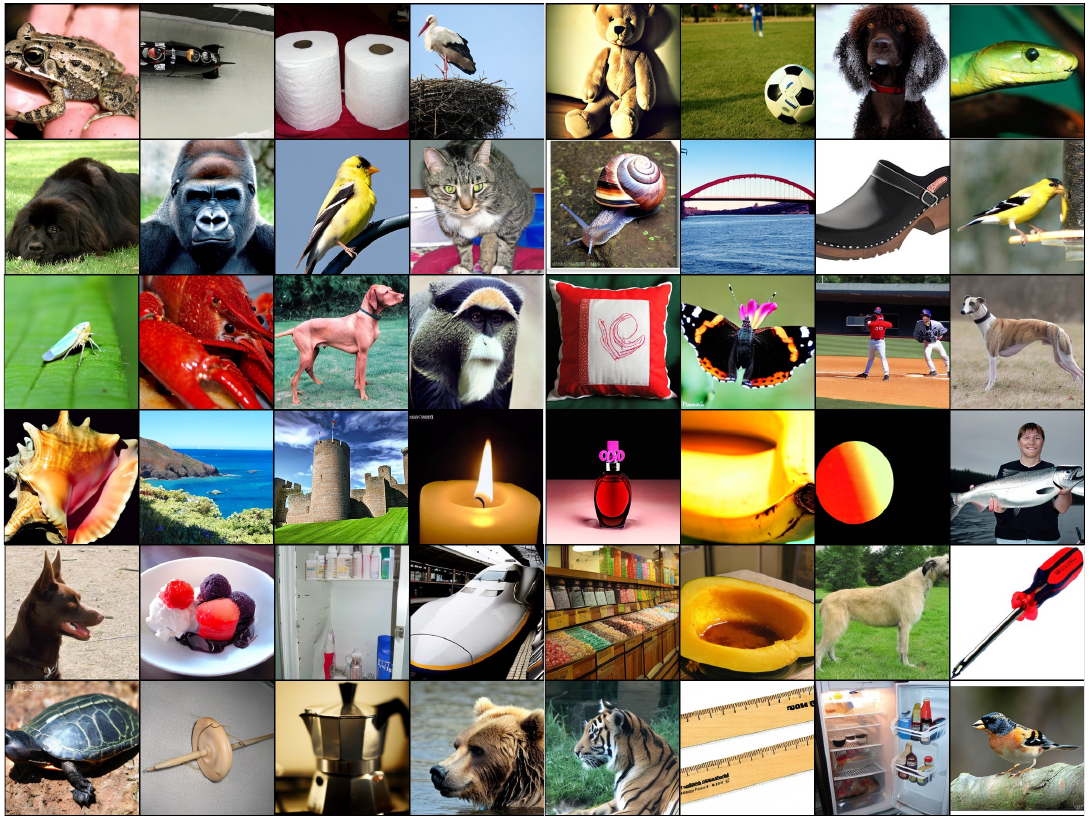}
    \caption{
    \textbf{Class-conditional ImageNet-256 qualitative samples.}
    Representative class-conditional generations on ImageNet-256.
    }
    \label{fig:imagenet256_qual}
\end{figure}

\begin{figure}[H]
    \centering
    {\scriptsize\bfseries
    \makebox[0.5\textwidth][c]{DiT}%
    \makebox[0.5\textwidth][c]{TPC-DiT}\par}
    \vspace{0.35em}
    \includegraphics[page=1,width=\textwidth]{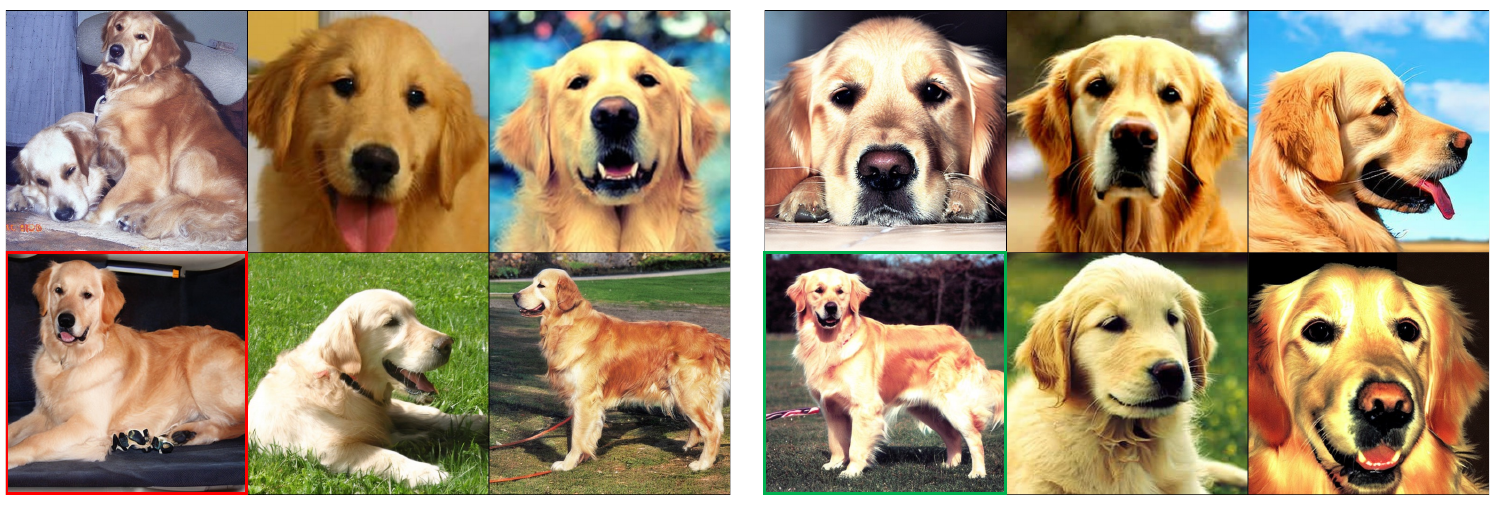}
    \caption{
    \textbf{Class-conditional ImageNet-512 qualitative samples.}
    Multi-page appendix figure with class-conditional ImageNet-512 comparisons; DiT samples are shown on the left and TPC-DiT samples are shown on the right.
    }
    \label{fig:qualitative_samples_appendix}
\end{figure}

\begin{figure}[H]
    \ContinuedFloat
    \centering
    {\scriptsize\bfseries
    \makebox[0.5\textwidth][c]{DiT}%
    \makebox[0.5\textwidth][c]{TPC-DiT}\par}
    \vspace{0.35em}
    \includegraphics[page=2,width=\textwidth]{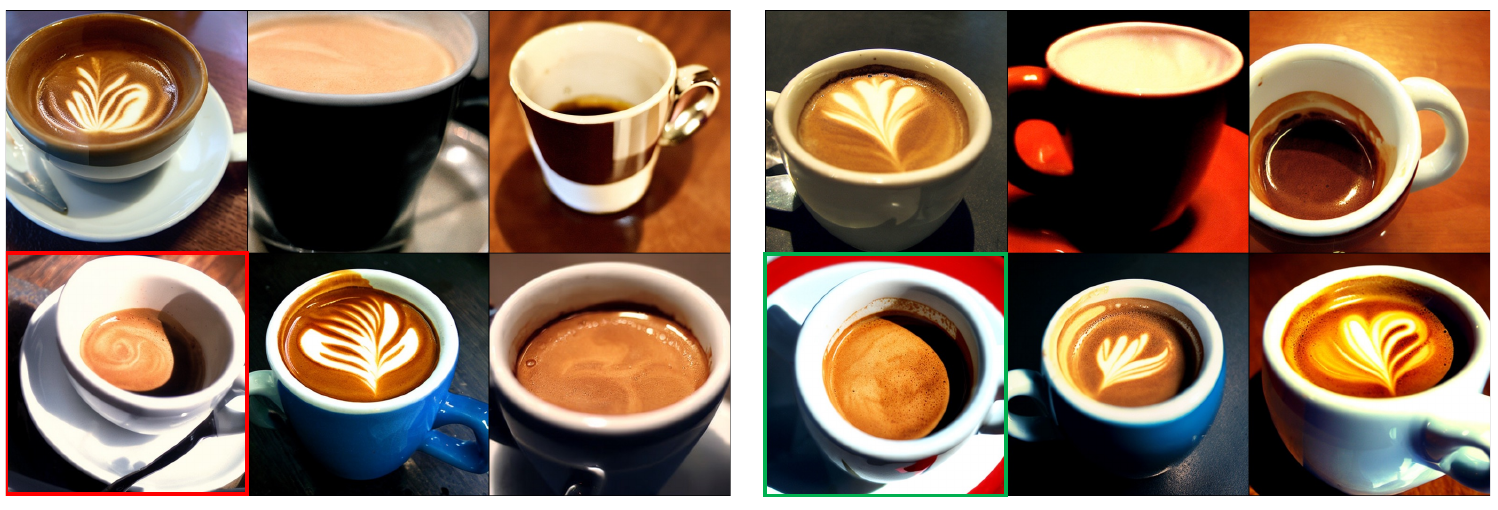}
    \caption[]{\textbf{Class-conditional ImageNet-512 qualitative samples.} Continued.}
    \label{fig:qualitative_samples_appendix_p2}
\end{figure}

\begin{figure}[H]
    \ContinuedFloat
    \centering
    {\scriptsize\bfseries
    \makebox[0.5\textwidth][c]{DiT}%
    \makebox[0.5\textwidth][c]{TPC-DiT}\par}
    \vspace{0.35em}
    \includegraphics[page=3,width=\textwidth]{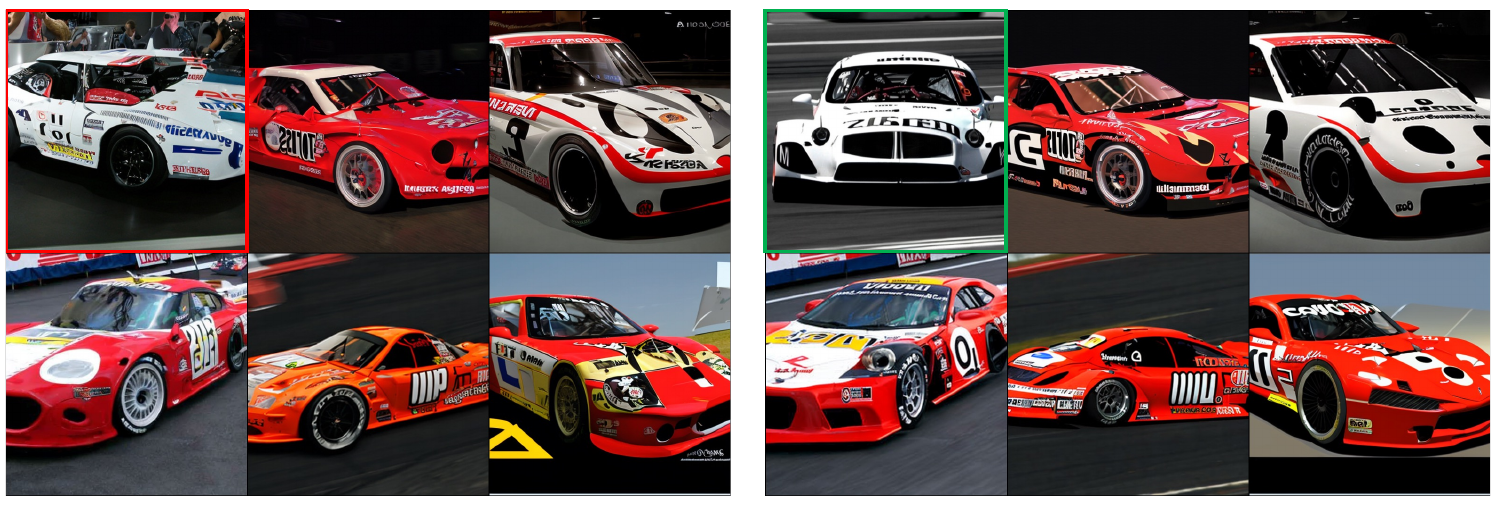}
    \caption[]{\textbf{Class-conditional ImageNet-512 qualitative samples.} Continued.}
    \label{fig:qualitative_samples_appendix_p3}
\end{figure}

\begin{figure}[H]
    \ContinuedFloat
    \centering
    {\scriptsize\bfseries
    \makebox[0.5\textwidth][c]{DiT}%
    \makebox[0.5\textwidth][c]{TPC-DiT}\par}
    \vspace{0.35em}
    \includegraphics[page=4,width=\textwidth]{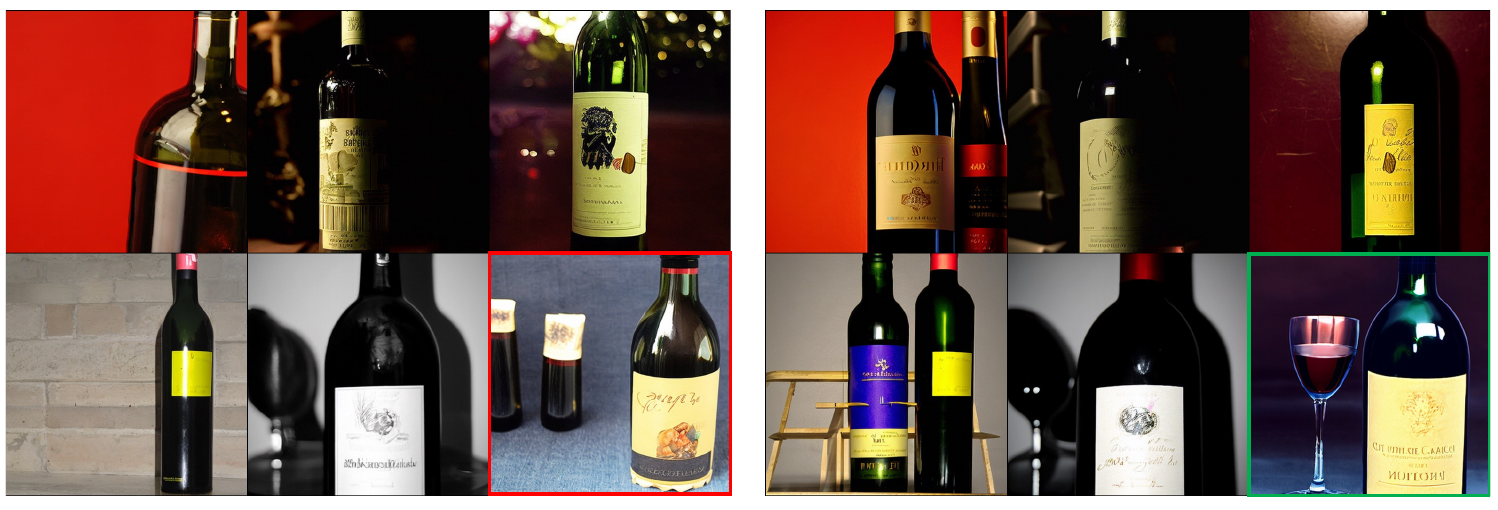}
    \caption[]{\textbf{Class-conditional ImageNet-512 qualitative samples.} Continued.}
    \label{fig:qualitative_samples_appendix_p4}
\end{figure}

\begin{figure}[H]
    \ContinuedFloat
    \centering
    {\scriptsize\bfseries
    \makebox[0.5\textwidth][c]{DiT}%
    \makebox[0.5\textwidth][c]{TPC-DiT}\par}
    \vspace{0.35em}
    \includegraphics[page=5,width=\textwidth]{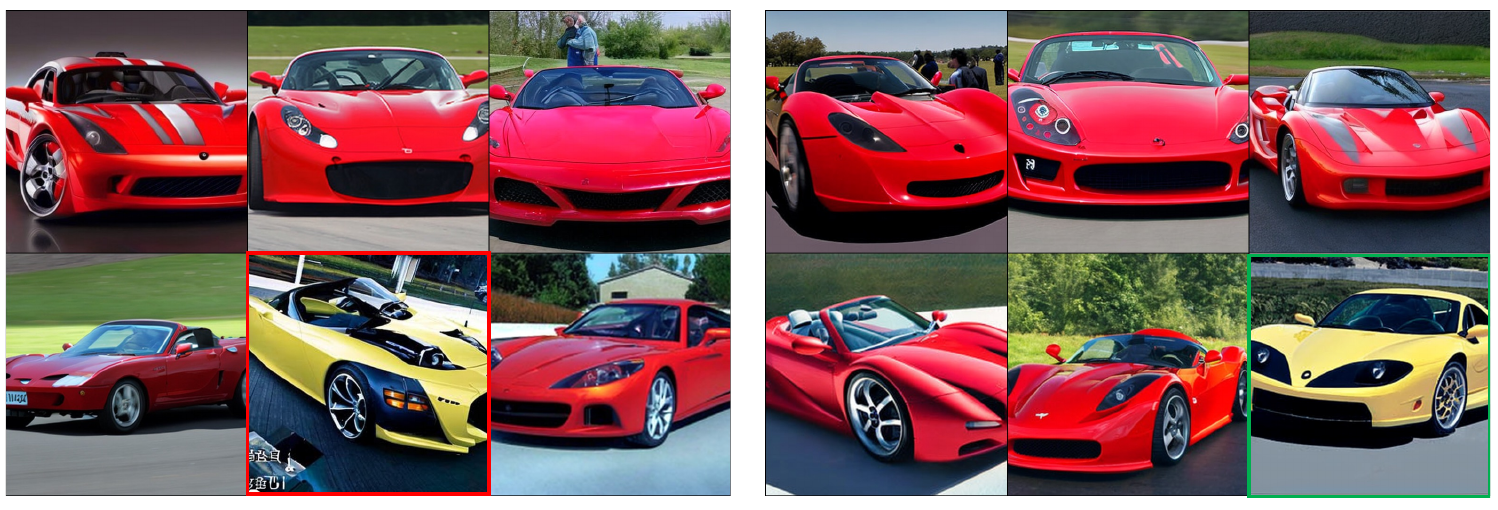}
    \caption[]{\textbf{Class-conditional ImageNet-512 qualitative samples.} Continued.}
    \label{fig:qualitative_samples_appendix_p5}
\end{figure}

\section{Impact Statement}
\label{app:impact_statement}

This paper presents work whose goal is to advance the field of machine learning. The method is a training-time modification for flow-matching-style generative models and does not introduce new data sources, deployment settings, or application-specific risks beyond those generally associated with generative modeling. We do not identify additional societal consequences that require special discussion.

\clearpage
\section*{NeurIPS Paper Checklist}

\begin{enumerate}

\item {\bf Claims}
    \item[] Question: Do the main claims made in the abstract and introduction accurately reflect the paper's contributions and scope?
    \item[] Answer: \answerYes{}
    \item[] Justification: The abstract and Introduction~\ref{sec:introduction} state the matched-cost empirical scope, and the claims are supported by Method~\ref{sec:method} and Experiments~\ref{experiments}. The paper does not claim formal theory or universal NFE reduction.

\item {\bf Limitations}
    \item[] Question: Does the paper discuss the limitations of the work performed by the authors?
    \item[] Answer: \answerYes{}
    \item[] Justification: The Conclusion includes a limitations statement on broader text-to-image and modality coverage, and Experiments~\ref{experiments} explicitly frame the matched-setting scope. The appendix further separates protocol details, diagnostics, and qualitative/contextual results.

\item {\bf Theory assumptions and proofs}
    \item[] Question: For each theoretical result, does the paper provide the full set of assumptions and a complete (and correct) proof?
    \item[] Answer: \answerNA{}
    \item[] Justification: The paper does not present formal theoretical results. The equations in Method~\ref{sec:method} define the training objective and pairing mechanisms used in the empirical study.

    \item {\bf Experimental result reproducibility}
    \item[] Question: Does the paper fully disclose all the information needed to reproduce the main experimental results of the paper to the extent that it affects the main claims and/or conclusions of the paper (regardless of whether the code and data are provided or not)?
    \item[] Answer: \answerYes{}
    \item[] Justification: Method~\ref{sec:method}, Experiments~\ref{experiments}, and Appendix Sections~\ref{app:experimental_protocol}--~\ref{app:supp_quant} specify the objective, pairing rules, algorithm, datasets, metrics, solver/NFE accounting, diagnostic measurements, and matched-cost protocol.

\item {\bf Open access to data and code}
    \item[] Question: Does the paper provide open access to the data and code, with sufficient instructions to faithfully reproduce the main experimental results, as described in supplemental material?
    \item[] Answer: \answerNo{}
    \item[] Justification: No public code release accompanies this paper. The manuscript nevertheless provides the algorithmic objective, evaluation protocol, and public benchmark setup needed to guide reproduction.

\item {\bf Experimental setting/details}
    \item[] Question: Does the paper specify all the training and test details (e.g., data splits, hyperparameters, how they were chosen, type of optimizer) necessary to understand the results?
    \item[] Answer: \answerYes{}
    \item[] Justification: Experiments~\ref{experiments} specifies datasets, metrics, solvers, NFE accounting, noise settings, and matched comparison protocol, while Appendix Section~\ref{app:experimental_protocol} gives the training algorithm and diagnostic definitions.

\item {\bf Experiment statistical significance}
    \item[] Question: Does the paper report error bars suitably and correctly defined or other appropriate information about the statistical significance of the experiments?
    \item[] Answer: \answerYes{}
    \item[] Justification: Diagnostic quantities in Table~\ref{tab:rf_mechanism_combined} and Appendix Table~\ref{tab:training_phase_diagnostics} include variability estimates.

\item {\bf Experiments compute resources}
    \item[] Question: For each experiment, does the paper provide sufficient information on the computer resources (type of compute workers, memory, time of execution) needed to reproduce the experiments?
    \item[] Answer: \answerYes{}
    \item[] Justification: Appendix~\ref{app:compute_hardware} reports the compute environment used for the reported experiments, including the internal GPU cluster, NVIDIA H200 GPUs with 141GB HBM3e memory per GPU, local high-throughput training storage, and the 32-GPU H200 configuration used for the largest ImageNet-$512{\times}512$ DiT-XL/2 run. It also notes that lower-resolution CIFAR-10 and ImageNet runs used fewer GPUs and that smaller pilot/tuning runs were conducted outside the final reported runs.
    
\item {\bf Code of ethics}
    \item[] Question: Does the research conducted in the paper conform, in every respect, with the NeurIPS Code of Ethics \url{https://neurips.cc/public/EthicsGuidelines}?
    \item[] Answer: \answerYes{}
    \item[] Justification: The work uses standard public image-generation benchmarks and does not involve human subjects, private data, surveillance, or personally identifying information. The study conforms to the NeurIPS Code of Ethics.

\item {\bf Broader impacts}
    \item[] Question: Does the paper discuss both potential positive societal impacts and negative societal impacts of the work performed?
    \item[] Answer: \answerYes{}
    \item[] Justification: The appendix includes an Impact Statement. The work is foundational generative-modeling research; it improves training of existing model classes and does not introduce new deployment settings, data sources, or application-specific risks beyond those generally associated with generative models.
    
\item {\bf Safeguards}
    \item[] Question: Does the paper describe safeguards that have been put in place for responsible release of data or models that have a high risk for misuse (e.g., pre-trained language models, image generators, or scraped datasets)?
    \item[] Answer: \answerNA{}
    \item[] Justification: The paper does not release a new pretrained image generator, scraped dataset, or other high-risk asset requiring release safeguards.

\item {\bf Licenses for existing assets}
    \item[] Question: Are the creators or original owners of assets (e.g., code, data, models), used in the paper, properly credited and are the license and terms of use explicitly mentioned and properly respected?
    \item[] Answer: \answerYes{}
    \item[] Justification: Existing datasets, baselines, and model families are credited through the corresponding citations in Experiments~\ref{experiments} and the references. The work uses standard public benchmarks and does not repackage or redistribute those assets.

\item {\bf New assets}
    \item[] Question: Are new assets introduced in the paper well documented and is the documentation provided alongside the assets?
    \item[] Answer: \answerNA{}
    \item[] Justification: The paper does not introduce or release a new dataset, benchmark, codebase, or pretrained model as a new asset.

\item {\bf Crowdsourcing and research with human subjects}
    \item[] Question: For crowdsourcing experiments and research with human subjects, does the paper include the full text of instructions given to participants and screenshots, if applicable, as well as details about compensation (if any)? 
    \item[] Answer: \answerNA{}
    \item[] Justification: The paper does not involve crowdsourcing experiments or research with human subjects.

\item {\bf Institutional review board (IRB) approvals or equivalent for research with human subjects}
    \item[] Question: Does the paper describe potential risks incurred by study participants, whether such risks were disclosed to the subjects, and whether Institutional Review Board (IRB) approvals (or an equivalent approval/review based on the requirements of your country or institution) were obtained?
    \item[] Answer: \answerNA{}
    \item[] Justification: The paper does not involve crowdsourcing or human-subject research, so IRB approval or equivalent human-subjects review is not applicable.

\item {\bf Declaration of LLM usage}
    \item[] Question: Does the paper describe the usage of LLMs if it is an important, original, or non-standard component of the core methods in this research? Note that if the LLM is used only for writing, editing, or formatting purposes and does \emph{not} impact the core methodology, scientific rigor, or originality of the research, declaration is not required.
    \item[] Answer: \answerNA{}
    \item[] Justification: LLMs are not an important, original, or non-standard component of the core method, experiments, or scientific claims. Any ordinary writing or editing assistance would not affect the methodology or originality of the research.

\end{enumerate}

\end{document}